\documentclass{article}

\usepackage[margin=3.5cm]{geometry}
\usepackage{flafter}

\usepackage{natbib}
\setcitestyle{open={(},close={)}}

\usepackage{soul}
\usepackage[british]{babel}
\usepackage{amsthm}
\usepackage{amsmath}
\usepackage{amssymb}
\usepackage{mathrsfs}

\usepackage{mathtools}
\usepackage{stmaryrd}

\usepackage{algpseudocode}

\usepackage{pgfplots}
\pgfplotsset{compat = 1.3}
\usetikzlibrary{shapes}
\usetikzlibrary{shapes.geometric}
\usetikzlibrary{patterns,shadings}
\usepgfplotslibrary{fillbetween}

\usepackage{tabularx}

\usepackage{wrapfig}
\newsavebox{\wrapfigbox}
               {
                 \begin{wrapfigure}{r}{#1}%
                   \begin{lrbox}{\wrapfigbox}%
                     \begin{minipage}{\dimexpr#1-2\fboxsep-2\fboxrule}%
                       \begin{center}%
                   } 
                  {%
                     \end{center}%
                     \end{minipage}%
                   \end{lrbox}%
                   \fbox{\usebox{\wrapfigbox}}
                 \end{wrapfigure}%
               }%

\usepackage{relsize}
\usepackage{pbox}

\newtheorem{theorem}{Theorem}
\newtheorem{lemma}[theorem]{Lemma}

\newtheorem{corollary}[theorem]{Corollary}

\newtheorem{remark}[theorem]{Remark}

\theoremstyle{definition}
\newtheorem{definition}[theorem]{Definition}
\newtheorem{example}{Example}

\usepackage{xcolor}
\definecolor{britishracinggreen}{rgb}{0.0, 0.26, 0.15}
\definecolor{brinkpink}{rgb}{0.98, 0.38, 0.5}

\usepackage{url}

\DeclareFontFamily{U}{jkpmia}{}
\DeclareFontShape{U}{jkpmia}{m}{it}{<->s*jkpmia}{}
\DeclareFontShape{U}{jkpmia}{bx}{it}{<->s*jkpbmia}{}
\DeclareMathAlphabet{\mathfrak}{U}{jkpmia}{m}{it}
\SetMathAlphabet{\mathfrak}{bold}{U}{jkpmia}{bx}{it}

\usepackage{enumitem}

\newlength{\myParindent}
\newcommand{\N}{\mathbb{N}}

\newcommand{\R}{\mathbb{R}}

\newcommand{\vareps}{\varepsilon}

\newcommand{\what}[1]{{\widehat{#1}}}
\newcommand{\mfrak}[1]{\mathfrak{#1}}
\newcommand{\mcal}[1]{\mathcal{#1}}

\newcommand{\abs}[1]{{\left| {#1} \right|}}

\newcommand{\card}[1]{{\abs{#1}}}
\newcommand{\lorder}[1][]{%
  \ifthenelse{\equal{#1}{}}%
             {{\le_{\alpha}}}%
             {{\le_{{#1}}}}%
}

\newcommand{\ind}{{\mcal{I}}}

\newcommand{\half}{{\frac{1}{2}}}

\DeclareMathOperator{\E}{{\mathbb{E}}}

\DeclareMathOperator{\sep}{sep}
\newcommand{\pairs}[1]{{[#1]}}
\newcommand{\splitarr}{{\xmapsto{{\,\,\! _T \,\!}}}}
\newcommand{\U}{\mcal{U}}
\newcommand{\Ultra}{\mfrak{U}}

\newcommand{\refines}{\sqsubset}

\newcommand{\btrees}[1]{{\mfrak{B}\!\left({#1}\right)}}

\newcommand{\cut}{\theta}
\newcommand{\val}{{\mathrm{val}}}
\newcommand{\cost}{{\mathrm{cost}}}
\newcommand{\mcc}{{\mathrm{jmp}}}

\newcommand{\method}[1]{{\textproc{#1}}}

\newcommand{\MaxDiCut}{\method{MaxDiCut}}
\newcommand{\DirSparsestCut}{\method{DirectedSparsestCut}}
\newcommand{\DirDensestCut}{\method{DirectedDensestCut}}
\newcommand{\SparsestCut}{\method{SparsestCut}}

\newcommand{\bound}{$O \big(\log^{3/2} \!\!\, n \big)$}
\newcommand{\minDeg}{{\mathrm{d}_{\mathrm{min}}}}
\newcommand{\hatval}{{\what{\mathrm{va}}\mathrm{l}}}
\newcommand{\HCOP}{{\mcal{HC}^{<\mcal{CL}}_{30,\vareps}}}
\newcommand{\HCCL}{{\mcal{HC}^{\mcal{CL}}}}

\newcommand{\ari}{{\mathrm{ARI}}}
\newcommand{\oari}{{\mathrm{\bar o ARI}}}
\newcommand{\loops}{{\mathrm{loops}}}

\newcommand{\deft}[1]{\textbf{\boldmath{#1}}}

\usepackage{tikz-cd}

\newcommand{\triline}{{%
\arrow[dd,-,crossing over,shift left=0]%
\arrow[dd,-,shift left=1.5,crossing over]%
\arrow[dd,-,shift right=1.5,crossing over]}}%
\newcommand{\lessline}{{\arrow[rruu,->,crossing over]}}

\newenvironment{acknowledgement}{{\flushleft \bf Acnowledgements.}}{}

\definecolor{DarkViolet}{HTML}{250077}
\usepackage[
  allbordercolors=white,
  colorlinks=true,
  allcolors=DarkViolet
]{hyperref}

\begin{document}

\title{An objective function for\\ order preserving hierarchical clustering}
\author{Daniel Bakkelund}
\date{\today}

\maketitle

\begin{abstract}
  We present a theory and an objective function for similarity-based hierarchical clustering
of probabilistic partial orders and directed acyclic graphs (DAGs). Specifically, given
elements $x \le y$ in the partial order, and their respective clusters $[x]$ and $[y]$,
the theory yields an order relation $\le'$ on the clusters such that $[x]\le'[y]$.
The theory provides a concise definition of order-preserving hierarchical clustering,
and offers a classification theorem identifying the order-preserving trees (dendrograms).
To determine the optimal order-preserving trees, we develop an objective function that
frames the problem as a bi-objective optimization, aiming to satisfy both the order
relation and the similarity measure. We prove that the optimal trees under the objective
are both order-preserving and exhibit high-quality hierarchical clustering.
Since finding
an optimal solution is NP-hard, we introduce a polynomial-time approximation algorithm
and demonstrate that the method outperforms existing methods for order-preserving
hierarchical clustering by a significant margin.

\end{abstract}


\section{Introduction}
\label{section:intro}
Clustering is one of the oldest and most popular techniques for exploratory data analysis and classification,
and methods for hierarchical clustering dates back almost a century. While clustering methods
for graphs and other types of structured data are common, it is also common that the
structure is lost during the clustering process; the clustering does not
retain the structure of the original data. Alternatively, the methods that do retain the structure
of the data are not easily combinable with a general notion of similarity to influence the
quality of the individual clusters.
In this article, we show how to do order preserving hierarchical clustering
for partially ordered data that is equipped with a measure of similarity between elements.
The goal is to produce hierarchical clusterings where the clusters
themselves are ordered to retain the original ordering of the data, while simultaneously
producing high quality clusters.

The objective function we present is an extension of the objective function for similarity based
hierarchical clustering by \citet{Dasgupta2016}, where we introduce an additional component
representing the \emph{level of comparability} according to the partial order.
This allows us to optimise on similarity and order preservation simultaneously, obtaining a hierarchical clustering
that balances the two objectives.

While several works exist that perform order preserving clustering, and although some of these
are also hierarchical in nature \citep{HerrmannEtAl2019}, there exists, to the authors' knowledge,
no attempts to provide a general definition of the concept.
As foundation for the presented objective, we therefore suggest a formal definition of order preserving
hierarchical clustering motivated by order theory and classical hierarchical clustering.

Since optimisation turns out NP-hard, we provide a polynomial time approximation algorithm
with a relative performance guarantee of \bound. We close the paper with a demonstration of the efficacy
of the approximation on the benchmark dataset \citep{BakkelundMachineParts}, showing that the
provided theory outperforms other methods with a large margin.

\medskip

{\flushleft The main contributions from this paper are:}

\begin{itemize}
\item A formal theory for order preserving hierarchical clustering.
\item An objective function for order preserving hierarchical clustering of partially ordered sets
  equipped with a similarity or dissimilarity.
\item A polynomial time algorithm approximating the objective function, with a relative performance guarantee
  of \bound.
\end{itemize}

\subsection{Motivating use case}
\label{section:motivation}
The theory described in this paper is inspired by an industry database of machine parts.
In the database, the machine parts are registered in \emph{part-of} relations, so that a part has a
reference to its containing part, according to the design.
The result is a family of \emph{part trees} where the nodes in
the trees are parts that consist of parts that consist of parts and so on.

Over time, some designs have been copied with small alterations, and due to business reasons,
all the parts of the new structure have been given new identifiers, with no reference to
where from it was copied.
In hindsight, it is desirable to discern which parts are equivalent to which other parts,
to allow for, for example, spare part interchange across designs.
We therefore seek a classification of part types into classes of interchangeable types.

Such a classification has to take into account the part-of relations:
As an example, assume that $a$, $b$, $c$ and $d$ are machine parts, and that
$a$ is a part of $b$ and $c$ is a part of $d$, as illustrated in Figure~\ref{fig:motivation}$.1$.
If we classify $a$ and $c$ as interchangeable (Figure~\ref{fig:motivation}$.2$),
then we also say that $a$ can be a part of $d$
and $c$ can be a part of $b$. It also implies that $a$ and $d$ can not be interchanged,
since you cannot replace a part by a sub-part.
Figure~\ref{fig:motivation} presents a selection of classifications of
these four elements, together with possible interpretations of what the classifications mean in the
domain of part-of relations.

\begin{figure}[htpb]
  \begin{center}
    \begin{tabular}{m{.12\textwidth}m{.3\textwidth}|m{.1\textwidth}m{.35\textwidth}}
      \begin{tikzpicture}[yscale=0.6]
        \node[draw,circle,minimum size=0.4cm,inner sep=0cm] (a) at (0,1) {$a$};
        \node[draw,circle,minimum size=0.4cm,inner sep=0cm] (b) at (1,1) {$b$};
        \node[draw,circle,minimum size=0.4cm,inner sep=0cm] (c) at (0,0) {$c$};
        \node[draw,circle,minimum size=0.4cm,inner sep=0cm] (d) at (1,0) {$d$};
        \draw[->] (a) -- (b);
        \draw[->] (c) -- (d);
      \end{tikzpicture}
      &
      $1.$ The original data. $a$ is a part of $b$ and $c$ is a part of $d$.
      &
      \begin{tikzpicture}[yscale=0.6]
        \draw[inner sep=0cm] (0,0.5) ellipse (.2cm and .6cm) {};
        \node (ea) at (0.1,0.65) {};
        \node (ec) at (0.1,0.35) {};
        \node (a) at (0,0.75) {$a$};
        \node (c) at (0,0.25) {$c$};
        \node[draw,circle,minimum size=0.4cm,inner sep=0cm] (b) at (1,1) {$b$};
        \node[draw,circle,minimum size=0.4cm,inner sep=0cm] (d) at (1,0) {$d$};
        \draw[->] (ea) -- (b);
        \draw[->] (ec) -- (d);
      \end{tikzpicture}
      &
      $2.$ $a$ and $c$ are equivalent, and both can be sub-parts of both $b$ and $d$, allowing
      for spare part interchange.
      \\
      \hline
      \begin{tikzpicture}[yscale=0.6]
        \draw[inner sep=0cm] (1,0.5) ellipse (.2cm and .6cm) {};
        \node (eb) at (0.9,0.65) {};
        \node (ed) at (0.9,0.35) {};
        \node (b) at (1,0.75) {$b$};
        \node (d) at (1,0.25) {$d$};
        \node[draw,circle,minimum size=0.4cm,inner sep=0cm] (a) at (0,1) {$a$};
        \node[draw,circle,minimum size=0.4cm,inner sep=0cm] (c) at (0,0) {$c$};
        \draw[->] (a) -- (eb);
        \draw[->] (c) -- (ed);
      \end{tikzpicture}
      &
      $3.$ $b$ and $d$ are equivalent, and both contain $a$ and $c$ as sub-parts.
      &
      \begin{tikzpicture}[yscale=0.6]
        \draw[inner sep=0cm] (0,0.5) ellipse (.2cm and .6cm) {};
        \node (a) at (0,0.75) {$a$};
        \node (c) at (0,0.25) {$c$};
        \draw[inner sep=0cm] (1,0.5) ellipse (.2cm and .6cm) {};
        \node (b) at (1,0.75) {$b$};
        \node (d) at (1,0.25) {$d$};
        \draw[->] (0.25,0.5) -- (0.78,0.5);
      \end{tikzpicture}
      &
      $4.$ $a$ and $c$ are equivalent, and $b$ and $d$ are equivalent. It is likely that
      one pair is a copy of the other.
      \\
      \hline
      \begin{tikzpicture}[yscale=0.6]
        \node[draw,circle,minimum size=0.4cm,inner sep=0cm] (a) at (0,0) {$a$};
        \node[draw,circle,minimum size=0.4cm,inner sep=0cm] (d) at (1.5,0) {$d$};
        \draw[inner sep=0cm] (.75,0) ellipse (.2cm and .6cm) {};
        \node (b) at (.75, 0.25) {$b$};
        \node (c) at (.75,-0.25) {$c$};
        \node (lhs) at (0.65,0) {};
        \node (rhs) at (.85,0) {};
        \draw[->] (a) -- (lhs);
        \draw[->] (rhs) -- (d);
      \end{tikzpicture}
      &
      $5.$ $b$ and $c$ are equivalent, and both $a$ and $b$ are sub-parts of $d$.
      $b$ and $c$ can be interchanged as spare parts in $d$.
      &
      \begin{tikzpicture}[yscale=0.6]
        \draw[inner sep=0cm] (0,0.5) ellipse (.2cm and .6cm) {};
        \node (a) at (0,0.75) {$a$};
        \node (d) at (0,0.25) {$d$};
        \draw[inner sep=0cm] (1,0.5) ellipse (.2cm and .6cm) {};
        \node (b) at (1,0.75) {$b$};
        \node (c) at (1,0.25) {$c$};
        \draw[->] (0.25,0.7) to [out=30,in=150] (0.78,0.7);
        \draw[<-] (0.25,0.3) to [out=-30,in=-150] (0.78,0.3);
      \end{tikzpicture}
      &
      $6.$ Not a valid clustering since the induced part-of relations are cyclic.
    \end{tabular}
    \caption{A selection of possible ordered clusterings of the set $\{a,b,c,d\}$.
      Possible interpretations in terms of the motivating use case are given together with the
      clusterings. All but $6)$ are examples of order preserving clusterings.
      In $6)$, the part-of relations constitute a cycle, implying that the parts are proper sub-parts of
      themselves, which is a contradiction.}
    \label{fig:motivation}
  \end{center}
\end{figure}

While each piece of machinery constitutes a tree of machine parts, some part types are used across
different types of machinery. As a result, the structure of machine part types and part-of relations make up
a directed acyclic graph, or equivalently, a partially ordered set.

In the sought after classification, a cluster should consist of interchangeable parts. But equally important,
the clusters should be ordered so that all the parts in the ``lesser'' cluster can be used
as a sub-part of any element of a ``greater'' cluster. This is depicted in
cases~$2.$ and~$4.$ of Figure~\ref{fig:motivation}.

This brings us to order preserving hierarchical clustering---a method that
provides us with a hierarchical clustering that preserves the partial order.
As a result, the corresponding ultrametric on the set of parts
reflects both the similarities of the parts, as well as the part-of relations.
And indeed, as we show in Section~\ref{section:ultrametric-order}, the more elements that are
strictly between two elements in the partial order, the larger the ultrametric distance between
the elements in an order preserving hierarchical clustering. This makes perfect sense
in the world of machine parts: if one part is deeply buried inside another part, such as a tire valve
on a bicycle, there is little chance for the parts to be interchangeable.

While flat clustering provides partitions that can be considered as classifications
of the data, hierarchical clustering provides tree structures over the data, where the
leaves are the original data points, and the bifurcations in the trees are where elements are joined
into successively larger clusters in the hierarchy.
In the context of machine learning, a significant benefit of these trees is that they
correspond to ultrametrics \citep{JardineSibson1971}.
This means that in addition to classification, hierarchical clustering yields a distance function on the data.
In the above use case, we are looking for items similar to a given item $x_0$. The ultrametric is ideal for
this purpose, allowing the user to explore items contained in successively larger neighbourhoods about $x_0$
until a suitable element is found.

\subsection{Related work} \label{section:related-work}
An earlier work on order preserving hierarchical clustering can be found in \citep{Bakkelund2021},
providing a means for clustering strictly partially ordered data. The work in this paper provides
significant generalisations of several key concepts:
\begin{itemize}
\item Since \citep{Bakkelund2021} requires the hard constraint $a<b \Rightarrow [a]<[b]$, it follows
  that comparable elements can never be clustered together, regardless of their similarity.
  Consider, for example, the sections of a book ordered by order of appearance.
  An order preserving hierarchical clustering
  will place similar sections together, while at the same time maintain the ordering of the sections. Hence,
  the order preserving hierarchical clustering can be seen to produce a table of contents for the book.
  When the constraint $a<b \Rightarrow [a]<[b]$ is applied, none of the sections can be clustered together,
  since they are all comparable. Hence, the only hierarchical clustering produced by \citep{Bakkelund2021}
  for linear orders is the singleton clustering.
  In this paper, we treat the order relation as non-strict, alleviating this problem.
\item A hard, binary constraint for order preservation, such as that of \citep{Bakkelund2021}, means that any data where
  the relation has cycles, for example due to errors in the data, cannot be clustered.
  In this paper, the order relation is considered probabilistic, and is weighted against the similarity measure.
  This allows us to cluster also data where the order relation is not a perfect partial order, providing a
  significant increase in availability and robustness in applications.
\item Finally, the theory in \citep{Bakkelund2021} only covers order preserving \emph{agglomerative} hierarchical clustering.
  This paper provides a general theory for order preserving hierarchical clustering that is independent
  of methods or algorithms, and covers all models.
\end{itemize}

The objective function presented in this paper extends on the work by \citet{Dasgupta2016}.
The model set forth by Dasgupta has been the subject of several publications, whereof
\citet{RoyPokutta2017}, \citet{MoseleyWang2017},
\citet{ChatziafratisEtAl2018},
\citet{CohenAddadEtAl2019}
are of particular interest with regards to this exposition.
Similarly to the method we present in this paper, the method by
\citet{ChatziafratisEtAl2018}
extends on Dasgupta's model in order to incorporate additional constraints from the domain.
Their method can be sorted in the category of
\emph{clustering with constraints} \citep{BasuDavidsonWagstaff2008,DavidsonRavi2005},
where the constraints can be seen as a particular type of structure imposed on the data. As such,
this is indeed a variation of what we could call \emph{structure preserving hierarchical clustering}.
However, the type of structural constraints do not cover order relations. And, while the constraints
in clustering with constraints are provided, for example, by domain experts in order to adjust the clustering,
the order relations used in order preserving hierarchical clustering emanate directly from the data.
It is the data itself that is ordered, and it is this order relation we wish to preserve.

\citet{CarlssonEtAl2014} present a hierarchical clustering
regime called \emph{quasi clustering} for asymmetric networks.
An asymmetric network is the same as a weighted directed graph, and is covered by what we define as
\emph{relaxed order relations} in this paper. We can therefore say that, at a high level, the
methods operate on similar data. However, the method presented by Carlsson et al.\ is almost
opposite of what we define as order preserving: objects are clustered together if there is significant
``flow'' between the elements. In order preserving clustering, this is exactly when elements are \emph{not}
placed together.

A field that does order preserving clustering is that of \emph{acyclic partitioning of graphs}.
Given a directed acyclic graph, the goal is to partition the vertices so that the graph that arises
when you preserve the edges going between partitions,
is also a directed acyclic graph; for an example, see the article
by \citet{NossackPesch2014branch}.
This is a topic that has been studied intensively with regards to applications ranging from
railway planning to parallel processing, compiler theory and VLSI. Hierarchical methods have also
been developed of late. While some of the new developments are of a more general
nature \citep{HerrmannEtAl2017,HerrmannEtAl2019}, a lot of the work focuses on satisfying domain
specific constraints and objectives. However, the methods do not easily combine with a similarity or
dissimilarity to support more classical notions of clustering.

Several works have been published on \emph{clustering of ordered data}.
We include two classes of methods in this group. The first is that of \emph{comparison based clustering},
where the degree of similarity of elements is derived from an order relation. An example is using
pairwise comparisons of elements done by users for preference ranking. See the quite recent article
by \citet{GhoshdastidarPerrotLuxburg2019} for more examples
and references. In this category, we also find the works of \citep{JanowitzBook2010},
providing a wholly order theoretic approach to hierarchical clustering, including the case where the
dissimilarity is replaced by a partially ordered set.
The other class is that of partitioning a family of ordered sets, so that sets that are similar to
each other are co-located in clusters. One example is the work by
\citep{KamishimaFujiki2003},
presenting a method for clustering sets containing preference data. Another example of methods in this
category is that of clustering of families of time series, such as the model described by
\citet{Luczak2016}, producing clusters consisting of similar times series.

\medskip

Whereas all the mentioned work touch upon one or more concepts involved in order preserving
hierarchical clustering, none of the methods offer a means to provide hierarchical clusterings where
the order relation and the similarity are combined, and where one seeks to find a hierarchical clustering
that attempts to satisfy both.

\subsection{Layout of the article}
Section~\ref{section:background} recalls the required background of graph theory, order relations
and hierarchical clustering, as well as recalling Dasgupta's cost function. Section~\ref{section:op}
gives a formal definition of order preserving hierarchical clustering, and precisely
identifies the binary trees over a set that correspond to order preserving hierarchical clusterings.
In Section~\ref{section:value-function}, we introduce our objective function.
Section~\ref{section:g} provides an investigation of the properties of the proposed model, looking
only at the effect of the order relation, keeping the similarity out of the equation.
The combined value function is the topic of Section~\ref{section:f}, where we view the clustering
problem in terms of bi-objective optimisation. Section~\ref{section:approx}
describes the polynomial time approximation algorithm, while Section~\ref{section:demo} provides
a demonstration of the efficacy of the method on data from the machine parts database described in
Section~\ref{section:motivation}.
Section~\ref{section:summary} provides a
concise summary and presents a short list of future research topics.

\subsection{Background}
\label{section:background}

A \deft{graph} is a pair $G=(V,E)$ of vertices and edges. Unless otherwise is specified, we assume that
all graphs are directed, and we refer to directed edges as arcs.
The graph $G$ is \deft{transitive} if, for every pair of elements $a,b \in V$ for which there is a path
from $a$ to $b$, there is an edge $(a,b) \in E$.
A \deft{graph isomorphism} $\theta : G \to G'$ for graphs $G=(V,E)$ and $G'=(V',E')$ is a bijection
$\theta : V \to V'$ for which $(x,y) \in E \Leftrightarrow (\theta(x),\theta(y)) \in E'$.

\medskip

For a set $X$, we write $A \subset X$ if $A$ is a subset of $X$, proper or not.
A \deft{binary relation} on a set $X$ is a subset $E \subset X \times X$.
A \deft{partial order} on a set $X$ is a binary relation $E$ on $X$
that is reflexive, antisymmetric and transitive, and in that case
we call the pair $(X,E)$ a \deft{partially ordered set}.
We usually denote the partially order by $\le$,
writing $x \le y$ for $(x,y) \in E$, and $x < y$ to mean $x \le y \land x \ne y$.
We write $\ind_\le$ to represent the indicator function of the partial order $\le$,
in which case we have $\ind_\le((x,y))=1$ if and only of $x \le y$, and zero otherwise.
For subsets $A,B \subset X$, we use the notation $\ind_\le(A,B)$ to denote the magnitude
\[
\ind_\le(A,B) = \sum_{(a,b) \in A \times B} \ind_\le((a,b)).
\]
We refer to two elements $x,y \in X$ as \deft{comparable} if either $x \le y$ or $y \le x$, and
a pair of elements that are not comparable are called \deft{incomparable}.
A \deft{linear order} is a partial order in which any two elements are comparable.
A \deft{chain} in a partially ordered set $(X,\le)$ is a subset
$Y \subset X$ that is linearly ordered with respect to $\le$.
Given two partial orders $\le,\le'$ on $X$, we say that $\le'$ is an \deft{extension} of $\le$ if
$x \le y \Rightarrow x \le' y$.
If in addition $\le'$ is a linear order, we call $\le'$ a \deft{linear extension} of $\le$.
For two order relations where one is an extension of the other, we say that the orders are \deft{compatible}.
The \deft{transitive closure} of a binary relation $E$ on $X$ is the smallest transitive relation
on $X$ that extends $E$.

A map $h : (X,\le) \to (Y,\le')$ between partially ordered sets is said to be \deft{order preserving}
if $x \le y \Rightarrow h(x) \le' h(y)$. Notice that a composition of order preserving maps is also
order preserving. See \citep{Shcroder2003book} for a general introduction to order theory.

\medskip

A (flat) \deft{clustering} of a set $X$ is a partition $\mcal{C}=\{C_i\}_{i=1}^k$
of $X$ into disjoint subsets.
We do not consider overlapping clusters; every clustering is a partition and corresponds to an
equivalence relation~$\sim$ on $X$. A \deft{cluster} is a block in the partition,
equivalently, an equivalence class under the equivalence relation. We employ the usual bracket
notation for equivalence classes based on representatives, writing $[x]_\mcal{C}$ for the cluster
of $x$ in the clustering $\mcal{C}$, possibly leaving out the subscript if there is no risk of ambiguity.
For every clustering $\mcal{C}$ of $X$, the unique map $q : X \to \mcal{C}$ defined by
$q(x) = [x]_\mcal{C}$, sending an element to its cluster,
is called the \deft{quotient map} of $\mcal{C}$.

\medskip

\newcommand{\tfgw}{.4\textwidth}
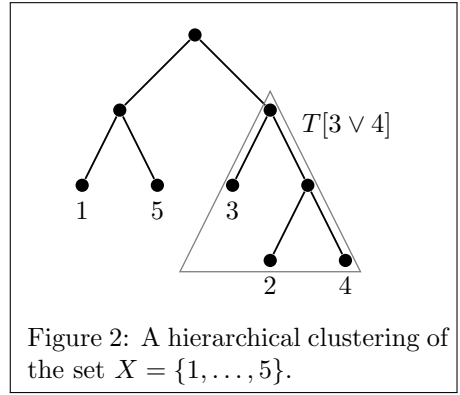
\begin{wrapfigure}{r}{\tfgw}
  \fbox{
    \begin{minipage}{\dimexpr\tfgw-2\fboxsep-2\fboxrule}
      \begin{center}
        \begin{tikzpicture}[
            xscale=0.5
          ]
          \node (spanTop) at (4,4.2) {};
          \node[circle,minimum size=5pt,fill=black,inner sep=0pt] (n1)  at (4,4) {};
          \node[circle,minimum size=5pt,fill=black,inner sep=0pt] (n21) at (2,3) {};
          \node[circle,minimum size=5pt,fill=black,inner sep=0pt] (n22) at (6,3) {};
          \node[circle,minimum size=5pt,fill=black,inner sep=0pt,label=below:{1}] (n31) at (1,2) {};
          \node[circle,minimum size=5pt,fill=black,inner sep=0pt,label=below:{5}] (n32) at (3,2) {};
          \node[circle,minimum size=5pt,fill=black,inner sep=0pt,label=below:{3}] (n33) at (5,2) {};
          \node[circle,minimum size=5pt,fill=black,inner sep=0pt] (n34) at (7,2) {};
          \node[circle,minimum size=5pt,fill=black,inner sep=0pt,label=below:{2}] (n41) at (6,1) {};
          \node[circle,minimum size=5pt,fill=black,inner sep=0pt,label=below:{4}] (n42) at (8,1) {};
          \draw[-,black,line width=.7pt] (n1) -- (n21) -- (n31);
          \draw[-,black,line width=.7pt] (n21) -- (n32);
          \draw[-,black,line width=.7pt] (n1) -- (n22) -- (n34) -- (n42);
          \draw[-,black,line width=.7pt] (n22) -- (n33);
          \draw[-,black,line width=.7pt] (n34) -- (n41);
          \draw[-,gray,line width=.5pt] (6,3.25) -- (8.4,0.85) -- (3.6,0.85) -- cycle;
          \node[anchor=south west] (join) at (6.6,2.5) {$T[3 \lor 4]$};
        \end{tikzpicture}
        \vspace{-.2cm}
        \caption{A hierarchical clustering of the set $X=\{1,\ldots,5\}$.}
        \label{fig:tree}
      \end{center}
  \end{minipage}}
\end{wrapfigure}%

Starting with a non-empty, finite set $X$, a \deft{split} of $X$ is a partition
$(A,B)$ of $X$ into two non-empty, disjoint subsets.
Each of these sets may again be split, giving
rise to a hierarchical decomposition of $X$ that can be drawn as a tree, where
every node has its split components as children until no more splits can be made.
This process produces a binary tree that has the set $X$ at the root, and the \deft{singleton sets}
$\{\{x\} \, | \, x \in X\}$ as leaf nodes, and defines a \deft{hierarchical clustering} of $X$.

We consider every split $(A,B)$ to be \emph{oriented}, meaning that the split $(A,B)$ is
considered to be different
from the split $(B,A)$. We refer to this concept as \deft{oriented splits},
and the trees as \deft{oriented trees}. If a node in a tree
splits into $(A,B)$, we refer to $A$ as the \deft{left child}, and $B$ as the \deft{right child} of the
parent node.
We denote the set of all oriented binary trees over $X$ that can be generated according to
the above procedure by $\btrees{X}$. And, for a tree $T \in \btrees{X}$, if $S$ is a node in $T$ that
splits into $(A,B)$, we denote this by writing $S \splitarr (A,B)$.

Notice that for $T \in \btrees{X}$, every node $S$ in $T$ is a subset of $X$. Let $x \lor y$
denote the smallest node in $T$, considered as a subset of $X$,
containing both $x$ and $y$, and let $T[x \lor y]$ be the
subtree of $T$ rooted at $x \lor y$.

In particular, $T[x \lor y]$ is the smallest subtree rooted
at an internal node containing both $\{x\}$ and $\{y\}$ as leafs.
Finally, we write $\card{T[x \lor y]}$ to denote \deft{the number of leaf nodes} of $T[x \lor y]$,
noting that this number coincides with the cardinality of the root node of $T[x \lor y]$,
namely $\card{x \lor y}$.
An example of a hierarchical clustering is depicted in Figure~\ref{fig:tree}.

\medskip

An \deft{ultrametric} on X is a metric $d$ on $X$ for which
\[
d(x,z) \le \max\!\left\{d(x,y),d(y,z)\right\} \quad \forall x,y,z \in X.
\]
We refer to the above equation as \deft{the ultrametric inequality}, and note that
it implies the triangle inequality. The set of all ultrametrics over $X$ is denoted~$\U(X)$.

Given a binary tree $T \in \btrees{X}$, \citet{RoyPokutta2017} show that the map
$\Ultra_X : \btrees{X} \to \U(X)$ given by
\begin{equation} \label{eqn:Ultra}
  \Ultra_X(T)(x,y) \ = \ \card{T[x \lor y]} - 1
\end{equation}
is an embedding for non-ordered trees. That is, every non-ordered binary tree over $X$ maps
to a unique ultrametric over $X$ via $\Ultra_X$.
Unless there is any chance of ambiguity, we denote the ultrametric $\Ultra_X(T)$
by $u_T$.

Given an ultrametric $u$ on $X$, for every $t \in \R_+$, the relation $\sim_t$ given by
$x \sim_t y \ \Leftrightarrow \ u(x,y) \le t$ is an equivalence relation on $X$,
and provides us with a flat clustering of $X$. (See \citet{CarlssonMemoli2010} or \citet{JardineSibson1971}
for details.)
When we refer to the clustering of an equivalence relation $\sim_t$ in the context of a
tree $T \in \btrees{X}$, unless otherwise is stated, this is always in relation to the ultrametric $u_T$.
In particular, for $T \in \btrees{X}$, \deft{the flat clusterings provided by $T$} are the
clusterings defined by the equivalence relations \hbox{$\{ \sim_t | \, t \in \R_+ \}$}. For example,
considering the tree in Figure~\ref{fig:tree} and the ultrametric $u_T$, the corresponding flat
clusterings are presented in Table~\ref{tab:flat-clusterings}.

\begin{table}[htpb]
  \begin{center}
    \begin{tabular}{l|l}
      Clustering & Equivalence relations \\
      \hline
      $\{1\},\{2\},\{3\},\{4\},\{5\}$ & $\{\sim_t |\, t \in [0,1) \}$ \\
        $\{1,5\},\{2,4\},\{3\}$ & $\{\sim_t |\, t \in [1,2) \}$ \\
          $\{1,5\},\{2,3,4\}$ & $\{\sim_t |\, t \in [2,4) \}$ \\
            $\{1,2,3,4,5\}$ & $\{\sim_t |\, t \in [4,\infty) \}$
    \end{tabular}
    \caption{The flat clusterings of the tree in Figure~\ref{fig:tree}. The clusters are
      listed in the left hand side column, and the sets of equivalence relations generating the clusterings
      are presented on the right.}
    \label{tab:flat-clusterings}
  \end{center}
\end{table}

\subsubsection{The Dasgupta cost model}

A \deft{similarity} on a set $X$ is a symmetric function $s : X \times X \to [0,1]$.
Given a similarity $s$ on $X$, the \deft{Dasgupta cost function} is the
function $\cost_s : \btrees{X} \to \R_+$ defined as
\begin{equation} \label{eqn:Dasgupta-cost}
\cost_s(T) \ = \ \sum_{\{x,y\}} \card{T[x \lor y]} s(x,y),
\end{equation}
where the sum is over all distinct pairs of elements of $X$.
The optimisation problem is to find a tree $T \in \btrees{X}$ minimising~\eqref{eqn:Dasgupta-cost}.
For a list of desirable properties of optimal trees under this model, please consult
the references given in the related work section.

Dasgupta defines similarity measures to have codomain all of $\R_+$,
but since minimising~\eqref{eqn:Dasgupta-cost} is invariant with respect to positive scaling of $s$,
and since $X$ is finite, the above definition implies no loss of generality.

\paragraph*{Dual formulation.}
The dual formulation of Dasgupta's cost function allows us to solve the optimisation
by maximisation. The motivation for introducing the dual is that when we introduce
order relations, it is easier to discuss the optimisation problem in the context of maximisation.

Define the dual of $s$ to be the function $s_d : X \times X \to [0,1]$ given by
\begin{equation} \label{eqn:sd}
s_d(x,y) \ = \ 1 - s(x,y),
\end{equation}
and define the \deft{value function} $\val_{s_d} : \btrees{X} \to \R_+$ by
\[
\val_{s_d} \ = \ \sum_{\{x,y\}} \card{T[x \lor y]} s_d(x,y).
\]
As of \citep[\S 4.1]{Dasgupta2016} any binary tree over $X$ maximising $\val_{s_d}$
is a tree that minimises~$\cost_s$.

\section{Order preserving hierarchical clustering} \label{section:op}

The contributions from this section are two-fold:
first, we provide a formal definition of order preserving hierarchical clustering.
Second, we define a class of binary trees called \emph{order preserving trees},
and prove that these trees are exactly the oriented binary trees over $X$ that correspond
to order preserving hierarchical clusterings.

Towards the end of the section, we present a result describing how the partial order is reflected in the
ultrametric $u_T$ of an order preserving tree $T$; that the more elements there are
between two elements in the partial order, the more different they are under the ultrametric.

\subsection{Formal definition}

We start by recalling the order theoretical notion of an order preserving flat clustering, and
then provide an extended definition that includes hierarchical clustering.

\medskip

The next two definitions and following theorem can be found in \citep[\S 3.1]{Blyth2005}. We recall them here
for completeness, and also to rephrase the concepts in the context of clustering.
We start by defining what it means for a flat clustering to be order preserving.

\begin{definition} \label{def:regular}
  Let $(X,\le)$ be a partially ordered set, and let $\mcal{C}$ be a clustering of $X$.
  We say that the clustering $\mcal{C}$ is \deft{order preserving (with respect to $\le$)}
  if there exists a partial order $\le'$ on $\mcal{C}$ so that
  \[
  x \le y \ \Rightarrow \ [x]_\mcal{C} \le' [y]_\mcal{C}.
  \]
\end{definition}

The next definition and theorem classify exactly the order preserving clusterings for
a partially ordered set.
\begin{definition} \label{def:induced-relation}
  Let $(X,\le)$ be an ordered set, and let $\mcal{C}=\{C_i\}_{i=1}^k$ be a clustering of $X$.
  Let $E$ be the binary relation on $\mcal{C}$ satisfying
  \[
  (C_i,C_j) \in E \ \Leftrightarrow \
  \exists x,y \in X \, : \, x \le y \ \land \ x \in C_i \ \land \ y \in C_j.
  \]
  We define \deft{the induced relation on $\mcal{C}$}, denoted $\le'$,
  to be the transitive closure of $E$.
\end{definition}

An instructive illustration of what the induced relation looks like,
is that of a \deft{$\mcal{C}$-fence}~\citep{Blyth2005}, or just fence, for short:
\begin{equation} \label{eqn:fence}
  \begin{tikzcd}
    b_1 \triline & & b_2 \triline      &                      & b_{n-1} \triline  & & b_n \triline \\
    & &                   & \cdots \arrow[ru,->] &                   & &              \\
    a_1 \lessline & & a_2 \arrow[ru,-] &                      & a_{n-1} \lessline & & a_n
  \end{tikzcd}
\end{equation}
Triple lines indicate elements in the same cluster, and the arrows represent comparability
in $(X,\le)$. The fence allows traversal from~$b_1$ to~$a_n$ along arrows and through clusters,
in which case we say that the fence \deft{links}~$b_1$ to~$a_n$.
The induced relation $\le'$ has the property that $x \le' y$
if and only if there exists a $\mcal{C}$-fence linking $x$ to $y$.

\begin{theorem}[{\citep[Thm.3.1]{Blyth2005}}] \label{thm:regular}
  Let $(X,\le)$ be an ordered set, and let $\mcal{C}$ be a clustering of~$X$.
  Then the following two statements are equivalent:
  \begin{enumerate}
  \item The induced relation $\le'$ is a partial order on $\mcal{C}$.
  \item The quotient map $p:(X,\le) \to (\mcal{C},\le')$ is order preserving. \label{item:op-quotient}
  \end{enumerate}
\end{theorem}

Notice that Item~\ref{item:op-quotient} of the theorem implies that $\mcal{C}$ is
an order preserving flat clustering.
Our task now, is to define what it means for a hierarchical clustering to be order preserving.
We start by recalling the definition of a hierarchical clustering, and then propose an extension that
includes order preservation.

For two clusterings $\mcal{C}=\{C_i\}_{i=1}^m$ and $\mcal{D}=\{D_j\}_{j=1}^n$ of $X$, we say that $\mcal{C}$ is a \deft{refinement} of $\mcal{D}$ if,
for every $C_i \in \mcal{C}$, there is a $D_j \in \mcal{D}$ so that $C_i \subset D_j$.
If this is the case, then the map $q:\mcal{C} \to \mcal{D}$ defined by $q(C)=D \Leftrightarrow C \subset D$
is also a quotient map.
We denote the refinement relation by $\mcal{C} \refines \mcal{D}$, and refer to the map $q$ as
\deft{the quotient map induced by the refinement}.

\medskip

A hierarchical clustering over $X$ is, thus,
a sequence $\{\mcal{C}_i\}_{i=0}^k$ of clusterings of $X$ for which
\begin{equation} \label{eqn:hc-chain}
  \mcal{C}_0 \refines \cdots \refines \mcal{C}_k.
\end{equation}
A sequence $\{\mcal{C}_i\}_{i=0}^k$ satisfying~\eqref{eqn:hc-chain} corresponds to
a \emph{persistent set} as of \citet{CarlssonMemoli2013}.
If we require that we have $\mcal{C}_k = \{X\}$, the sequence corresponds to a
\emph{dendrogram} in the nomenclature of \citet{JardineSibson1971},
and, recalling the singleton partition $S(X)=\cup_{x \in X} \{\{x\}\}$,
if $\mcal{C}_0=S(X)$, the sequence corresponds to a \emph{definite dendrogram}, also according to
Jardine and Sibson.
All of the mentioned concepts define hierarchical clustering in the classical sense.

\bigskip

We now have the following observation.
\begin{theorem} \label{thm:op-hc}
  Given a partially ordered set $(X,\le)$ and a hierarchical clustering $\mcal{H}=\{\mcal{C}_i\}_{i=0}^k$
  of~$X$, then the following statements are equivalent:
  \begin{enumerate}
  \item All quotient maps $p_i : X \to \mcal{C}_i$, for $0 \le i \le k$, are order
    preserving with respect to $\le$; \label{item:def}
  \item All quotient maps $q_i : C_i \to C_{i+1}$ for $0 \le i \le k-1$, induced by the refinements,
    are order preserving with respect to the induced order relations; \label{item:thm}
  \item The following diagram is commutative and all quotient maps are order preserving:
    \[
    \begin{tikzcd}
      \mcal{C}_0 \ar[r,"q_0"] &
      \mcal{C}_1 \ar[r,"q_1"] &
      \cdots \ar[r,"q_{k-1}"] &
      \mcal{C}_k \\[-1em]
      & \cdots \\
      X
      \ar[uu,"p_0",pos=0.6]
      \ar[uur,"p_1",pos=0.5,outer sep=-1pt]
      \ar[uurrr,"p_k",pos=0.55,outer sep=-1pt] 
    \end{tikzcd}.
    \] \label{item:diagram}
  \end{enumerate}
\end{theorem}

{\flushleft Before} presenting the proof, we suggest the following definition:
\begin{definition}
  Given a partially ordered set $(X,\le)$, \deft{an order preserving hierarchical
    clustering of $X$ with respect to $\le$} is a hierarchical clustering
  of $X$ satisfying any and all of the statements of Theorem~\ref{thm:op-hc}.
\end{definition}

\begin{proof}[Proof of Theorem~\ref{thm:op-hc}]
  We start by proving the equivalence of of statements~\ref{item:def} and~\ref{item:thm}.
  Assume first that the quotient maps $p_i : X \to \mcal{C}_i$ are order preserving.
  If we can show that for every $A,B \in \mcal{C}_i$ we
  have $A \le'_i B \Rightarrow q_i(A) \le'_{i+1} q_{i}(B)$, we are done.

  Since $A \le'_i B$, there are elements $a \in A$ and $b \in B$
  that are linked via a $\mcal{C}_i$-fence. Now, $\mcal{C}_i$ is a refinement of $\mcal{C}_{i+1}$, so
  every pair of co-clustered elements in $\mcal{C}_i$ are also co-clustered in $\mcal{C}_{i+1}$. Since this
  implies that all in-cluster links in fences are maintained, there is also a $\mcal{C}_{i+1}$-fence
  linking $a$ and $b$, and we must therefore have $A \le'_{i+1} B$. Since $p_{i+1}$ is order preserving,
  $\le'_{i+1}$ is an order relation on $\mcal{C}_{i+1}$,
  so according to Theorem~\ref{thm:regular}, $q_i$ is order preserving too.

  Now assume that all the $q_i$ are order preserving.
  Since the maps $p_i$ must be quotient maps sending elements to their clusters,
  there is only one possible definition of these maps; namely
  \begin{equation*} 
    p_i \ = \ q_{i-1} \circ \cdots \circ q_0 \circ p_0.
  \end{equation*}
  Since $p_0$ is order preserving, the result follows by induction on $i$, as all maps
  on the right hand side are order preserving, and since a composition of order preserving
  maps is an order preserving map.

  Finally, the diagram commutes due to the above definition of the $p_i$, and the maps are all
  order preserving if and only of both statements~\ref{item:def} and~\ref{item:thm} hold.
\end{proof}

\subsection{Order preserving binary trees}

We now define order preserving binary trees and show that these trees are exactly the binary trees that
correspond to order preserving hierarchical clusterings.

All the concepts on binary trees given in this section extend straight forwardly to the
class of all oriented trees over $X$, regardless of the arity of the splits.
However since the rest of the paper deals exclusively with binary trees, we have chosen to
stick to binary trees also for this part.

Recall that for a split $(A,B)$, the order of the components is significant;
the split is ordered.
\begin{definition} \label{def:order-preserving-tree}
  Let $(X,\le)$ be a partially ordered set.
  A split $(A,B)$ of $X$ is \deft{order preserving (with respect to $\le$)} if
  \begin{equation} \label{eqn:order-preserving}
  x \le y \ \Rightarrow \ (y,x) \not \in A \times B \qquad \forall x,y \in X.
  \end{equation}
  An \deft{order preserving tree} is a tree in which every
  split is order preserving.
\end{definition}
The idea is that, in an order preserving split there are no ``reversals'' of ordered pairs; the
order of the elements shall coincide with the order of the sets they reside in.
For example, if $x \le y$, we can have $(x,y) \in A \times B$, $x,y \in A$ or $x,y \in B$.
The only illegal constellation is that of equation~\eqref{eqn:order-preserving}.

Recalling the indicator function $\ind_\le$ of the partial order,
we could replace equation~\eqref{eqn:order-preserving} by requiring that $\ind_\le(B,A)=0$,
a statement that is equivalent to that of Definition~\ref{def:order-preserving-tree}.

\medskip

Notice that if we draw a tree $T \in \btrees{X}$ on a piece of paper, with the root at the top and
the leaves at the bottom, since all splits are oriented, the tree induces a unique
linear ordering on the leaves as they appear from left to right.
We write $\le_T$ to denote this ordering, and refer to $\le_T$ as
\deft{the linear order on $X$ induced by $T$}.

\medskip

For an order preserving tree,
the order $\le_T$ is a linear extension of the partial order:
\begin{lemma} \label{lemma:preserved-leaf-order}
  If $(X,\le)$ is a partially ordered set, then $T \in \btrees{X}$ is an order preserving tree with
  respect to $\le$ if and only if
  \[
  x \le y \ \Rightarrow \ x \le_T y \quad \forall x,y \in X.
  \]
\end{lemma}
\begin{proof}
  Assume first that $T$ is order preserving, and pick $x,y \in X$ for which $x \le y$.
  Let $S = x \lor y$ and $S \splitarr (A,B)$, so that $x$ and $y$ end up in different split components
  of $S$. Since $T$ is order preserving, we must have $x \in A$ and $y \in B$, but this means that
  $x \le_T y$ too, so the implication holds.
  
  For the opposite direction, let $S \splitarr (A,B)$ be a split in $T$.
  Furthermore, let $a,b \in X$ be comparable, and let $a$ and
  $b$ end up in different components in the split of $S$; that is: $a \le b$, and $a \lor b = S$.
  Due to the assumption $a \le b \Rightarrow  a \le_T b$, we get $a \in A$ and $b \in B$. Since this
  covers all comparable pairs across $A$ and $B$, the split $S \splitarr (A,B)$ is order preserving.
\end{proof}

Recall that for a tree $T \in \btrees{X}$ and a non-negative real number $t$,
there is a flat clustering $\mcal{C}_t$ of $X$ defined by the equivalence
relation $\sim_t$ (Section~\ref{section:background}).
For every flat clustering under a tree $T$, the relation $\le_T$ also induces
an order on the clusters:
\begin{lemma} \label{lemma:induced-cluster-order}
  For a partially ordered set $(X,\le)$, let $T \in \btrees{X}$ and $u_T = \Ultra_X(T)$. If $t \in \R_+$, and 
  if $\mcal{C}_t=\{C_i\}_{i=1}^{k}$ are the clusters of $\sim_t$,
  then the enumeration on the clusters can be chosen so that
  \begin{equation} \label{eqn:induced-cluster-order}
  x \le_T y \ \land \ x \in C_i \ \land \ y \in C_j \ \Rightarrow \ i \le j \qquad \forall x,y \in X.
  \end{equation}
\end{lemma}
\begin{proof}
  We prove this by induction. Let $\{t_1 > \cdots > t_m\}$ be a maximal set of real of numbers
  so that each equivalence relation $\sim_{t_i}$ is distinct for $1 \le i \le m$.
  Clearly, the statement holds for $t_1$, where there is only one cluster.
  Assume that the statement holds for $t_k$ when $k \ge 1$, and consider the case $t_{k+1}$.
  Let $C_i,C_j$ be distinct clusters under $\sim_{t_{k+1}}$.
  If there is a cluster $C$ under $\sim_{t_k}$ for which
  $C \splitarr (C_i,C_j)$, then the split defines an ordering of $C_i$ and $C_j$ that is compatible
  with $\le_T$.
  If no such $C$ exists, then $C_i$ and $C_j$ are subsets of distinct clusters under $\sim_{t_k}$.
  Due to the induction hypothesis, these equivalence classes are already ordered compatibly
  with $\le_T$, and this order propagates to $C_i$ and $C_j$, maintaining compatibility.
\end{proof}

We accept the risk of using $\le_T$ also to denote the relation on the clusters, writing $C_i \le_T C_j$.

\begin{remark}
  In order theoretic jargon, Lemma~\ref{lemma:induced-cluster-order} implies that
  all the clusters are convex with respect to $\le_T$;
  if $x,z \in C_i$ and $x \le_T y \le_T z$, then $y \in C_i$ too.
\end{remark}

We are now ready to state the main result of this section.

\begin{theorem} \label{thm:order-preserving-trees}
  Let $(X,\le)$ be a partially ordered set, and let $T \in \btrees{X}$.
  Then $T$ is an order preserving tree with respect to $\le$ of and only if
  the hierarchical clustering corresponding to $T$ is an order preserving hierarchical clustering
  of $X$ with respect to $\le$.
\end{theorem}
\begin{proof}
  Recall that the hierarchical clustering corresponding to $T$ is the sequence of clusters
  $\mcal{H}=\{\mcal{C}_i\}_{i=0}^k$ corresponding to the equivalence relations $\sim_t$ under $u_T$.

  Assume that $T$ is order preserving, and let $\mcal{C}=\{C_i\}_{i=1}^m$ be a flat clustering in $\mcal{H}$.
  We shall show that $\mcal{C}$ is an order preserving clustering, implying that the quotient map
  $p : X \to \mcal{C}$ is order preserving, where after the theorem follows from Theorem~\ref{thm:op-hc}.
  
  By combining Lemmas~\ref{lemma:preserved-leaf-order} and~\ref{lemma:induced-cluster-order},
  we can enumerate the clusters of $\mcal{C}$ according to $\le_T$ so that $C_1 \le_T \cdots \le_T C_m$,
  and moreover, $x \le y \Rightarrow [x]_\mcal{C} \le_T [y]_\mcal{C}$.
  Let $\le'$ be the induced relation on~$\mcal{C}$ (Definition~\ref{def:induced-relation}).
  Then $x \le y \Rightarrow [x]_\mcal{C} \le' [y]_\mcal{C}$.
  But this means that the linear order $\le_T$ on $\mcal{C}$ is an
  extension of $\le'$, so $\le'$ must be a partial order on $\mcal{C}$.

  \medskip

  For the only-if part, let $\mcal{H}$ be the clustering corresponding to $T$, and assume that $\mcal{H}$ is
  an order preserving hierarchical clustering of $X$ with respect to $\le$.
  In particular, we have $C_0 = S(X)$ as an order preserving clustering, and since
  $\le_0$ has $\le_T$ as a linear extension, it follows that $\le_T$ is a linear extension also of $\le$.
  According to Lemma~\ref{lemma:preserved-leaf-order}, this means that $T$ is order preserving with
  respect to $\le$.
\end{proof}

Now we know that we are looking for order preserving trees. What remains is simply to
device a method to identify the best of them.
Before we close the section, we present a result that links order preserving trees to ultrametrics,
and shows how the partial order is reflected in the ultrametric distances.

\subsection{Order preserving trees and ultrametric distances} \label{section:ultrametric-order}
We now demonstrate that for an order preserving tree $T \in \btrees{X}$ over a partially
ordered set $(X,\le)$, given the ultrametric $u_T=\Ultra_X(T)$, the distance between two elements $u_T(x,y)$
is bounded below by the number of elements strictly between $x$ and $y$ in $\le$.

\medskip

For an ordered set $(X,\le)$ and $x,y \in X$, define the function $\mcc : X \times X \to \N$ so that
if $x \le y$, then $\mcc(x,y)$ is one less than the cardinality of a maximal chain having $x$ as minimal
element and $y$ as maximal element, and zero  if $x \not \le y$. The magnitude of $\mcc(x,y)$
is equal to the maximal
number of ``jumps'' one would have to make, jumping one element at the time, starting at~$x$ and stopping
at~$y$, when jumping only in the direction of strictly larger elements.
We define the \deft{ordered separation} of $x,y \in X$ to be the magnitude
\[
\sep(x,y) = \max\{\mcc(x,y),\mcc(y,x)\}.
\]
Notice that $\sep(x,y)=0$ if and only of either $x=y$, or $x$ and $y$ are incomparable.

\begin{theorem} \label{thm:ultrametric-order}
  Let $(X,\le)$ be a partially ordered set, let $T \in \btrees{X}$ be order preserving
  with respect to $\le$, and let $\{x_i\}_{i=1}^n$ be the enumeration of elements of $X$ corresponding
  to the order $\le_T$ so that $i \le j \Leftrightarrow x_i \le_T x_j$. Let $u_T = \Ultra_X(T)$. Then
  \[
  u_T(x_i,x_j) \ge \abs{i-j} \ge \sep\!\left( x_i, x_j \right).
  \]
\end{theorem}
\begin{proof}
  The right hand side inequality follows from the fact that every element
  between $x_i$ and $x_j$ under $\le$, must also be between $x_i$ and $x_j$ under $\le_T$, due to
  Lemma~\ref{lemma:preserved-leaf-order}.
  The left hand side inequality follows from the definition of $u_T$,
  namely $u_T(x_i,x_j) = \card{T[x_i \lor x_j]} -1$, and the fact that  every
  leaf between $x_i$ and $x_j$ under $\le_T$ must be a leaf of $T[x_i \lor x_j]$.
  To see why this must be true, it is sufficient to consider the planar drawing of $T$, and realising
  that every element between $x_i$ and $x_j$ must join either $x_i$ or $x_j$ before $x_i$ and $x_j$
  are joined.
\end{proof}

This essentially tells us that the more elements that lie between two elements, the more different
they are under an ultrametric that corresponds to an order preserving tree.
Our choice of ultrametric definition makes the result easily quantifiable, but the observation holds
for all ultrametrics. The intuition behind this is straight forward:
The more elements that are between two elements in the partial order, the more elements there are
between the elements in $\le_T$. And for any ultrametric (equivalently, dendrogram) that is based on an
order preserving tree, this means that the more elements there are between two elements in the partial
order, the higher in the tree (or dendrogram) they join, leading to a higher ultrametric distance.

\bigskip

Looking back at the motivating industry use case, where we have parts that consists of parts
that consists of parts and so on, Theorem~\ref{thm:ultrametric-order} simply reflects the fact
that the more nested levels of composition there is between two parts, the less similar they are.

\section{An objective function for trees over ordered data}
\label{section:value-function}

We now turn our attention to the task of devising a method
for identifying ``good'' order preserving trees.
In this section, we define our representation of ordered data, and present a value function
for binary trees that incorporates the idea of order preservation. And, as we show in the next section,
the function is a realisation of order preserving hierarchical clustering in the sense that it correctly
identifies order preserving trees when the data is partially ordered.

\medskip

In the continuation, we treat the order relation as a stochastic object.
Let the function $\omega : X \times X \to [0,1]$ represent a family of random binary relations $E$ on $X$,
where $(x,y) \in E$ with probability $\omega(x,y)$.
We refer to $\omega$ as a \deft{relaxed binary relation} on $X$.

\begin{example} \label{ex:random-family}
If $(X,\le)$ is a partially ordered set, and if $0 \le q < p \le 1$, we can define
\[
\omega(x,y) =
\begin{cases}
  p & \text{if $x \le y$}, \\
  q & \text{otherwise}.
\end{cases}
\]
Then $\omega$ is a relaxed binary relation.
Moreover, for any random relation $E$ in the family of relations
represented by $\omega$, recalling the indicator function $\ind_E$ of $E$,
the expected value of $\ind_E$ is  $\E\!\left( \ind_E(x,y) \right) = \omega(x,y)$.
However, we cannot in general expect $E$ to be a partial order on $X$.
\end{example}

To remind ourselves that we are working
with ordered sets, we shall mostly refer to $\omega$ as a \deft{relaxed order}, keeping in mind
that it is the same thing as a relaxed binary relation.

\medskip

Next, define \deft{the antisymmetrisation of $\omega$} as
the function $g : X \times X \to [-1,1]$, given by
\begin{equation} \label{eqn:g}
g(x,y) \ = \ \omega(x,y) - \omega(y,x).
\end{equation}
As the name suggests, the function is antisymmetric,
and it computes the \emph{signed net comparability} between~$x$ and~$y$.
We see that $g(x,y)$ takes on a large positive value if there is strong evidence for $x < y$,
and a large negative value if there is strong evidence for $x > y$.
Moreover, if the comparability is ambiguous, with $\omega(x,y)$ and $\omega(y,x)$ being
very similar, $g(x,y)$ will take on a value close to zero.

To see why this is useful, consider Example~\ref{ex:random-family}, and assume
that $q$ and $p$ are very close in magnitudes, meaning that
the probability of $(x,y) \in E$ is very close to the probability of $(y,x) \in E$.
If we are producing a split $(A,B)$ of $X$ and trying to decide where to place $x$ and $y$ in order
to make the split order preserving (Definition~\ref{def:order-preserving-tree}),
it makes little sense to strongly favour one placement over the other in this situation.
The close to zero value of $g(x,y)$ reflects this fact.
The function $g$ can be used for hierarchical clustering on its own
without any similarity, and the properties of $g$ is the subject of study of Section~\ref{section:g}.

\medskip

However, to achieve order preserving hierarchical clustering, we must combine $g$ with a similarity.
Define an \deft{ordered similarity space} to be a triple $(X,s,\omega)$, where $X$ is a set,
$s$ is a similarity, and $\omega$ is a relaxed order on $X$.
For a given $(X,s,\omega)$, recall the similarity dual $s_d$ from~\eqref{eqn:sd}.
Let \deft{the split value function} $f : X \times X \to [-1,2]$ be defined as
\begin{equation} \label{eqn:f}
  f(x,y) \ = \ s_d(x,y) + g(x,y).
\end{equation}
This function will attain its maximal value on $(x,y)$ if both $s_d(x,y)$ and $g(x,y)$ are maximal,
meaning that $x$ and $y$ are considered to be both highly dissimilar and we have $x < y$ with high
confidence; both being evidence for splitting $x$ and $y$ into separate clusters.

On the other hand, the function will attain its minimal value on $(x,y)$ only if $g(x,y)$ has a large
negative value. Since $s_d$ is symmetric and $g$ is antisymmetric, this means that
by swapping the arguments, $f(y,x)$ will attain a large positive value, again being evidence
for splitting the elements apart.

And finally, if the elements are not comparable, then $g(x,y)$ is close to zero,
and if the elements are not dissimilar, then $s_d$ is close to zero, so $f$ is close to zero
if there is little evidence for splitting the elements apart.

\begin{definition} \label{def:maximisation-problem}
  Given $(X,s,\omega)$ and a tree $T \in \btrees{X}$, the \deft{value of $T$ under $(X,s,\omega)$}
  is given by the function $\val_f : \btrees{X} \to \R$, defined as
  \begin{equation} \label{eqn:val-f}
    \val_f(T) \ = \ \sum_{x \le_T y} \card{T[x \lor y]} f(x,y),
  \end{equation}
  where the iteration order is dictated by the linear order $\le_T$ induced by $T$.
  Furthermore, an \deft{optimal hierarchical clustering of $(X,s,\omega)$} is a binary tree
  $T^* \in \btrees{X}$ for which
  \[
  \val_f(T^*) = \max_{T \in \btrees{X}} \val_f(T).
  \]
\end{definition}

The multiplier $\card{T[x \lor y]}$ of $\val_f$ encourages pairs of elements for which $f(x,y)$ attains a large
positive value to be split apart close to the root of the tree.
Looking back at the discussion preceding the definition, this means that elements that are strongly indicated
to belong to different clusters will be split apart close to the root.

But the maximisation also favors an orientation on elements: If $\omega$ is as in
Example~\ref{ex:random-family}, and if $x < y$, then $g(x,y) > 0$ and $g(y,x) < 0$.
This intuitively suggests that a tree $T \in \btrees{X}$ in which $x \le_T y$ will have higher value
compared to if $y \le_T x$, and therefore suggests that a maximal tree will be order preserving,
as of Lemma~\ref{lemma:preserved-leaf-order}; a property that will be proven formally
in Section~\ref{section:g}.

\medskip

There is an alternative formulation of~\eqref{eqn:val-f} which we will make occasional use of.
Recall that every node in a tree $T \in \btrees{X}$ is a subset of $X$, and an internal node $S \in T$ splits
according to $S \splitarr (A,B)$. If we define $f(A,B) = \sum_{(a,b) \in A \times B} f(a,b)$,
we can formulate $\val_f$ as
\begin{equation*}
  \val_f(T) \ = \ \sum_{S \splitarr (A,B)} \card{S} f(A,B),
\end{equation*}
where the sum is over all the splits $S \splitarr (A,B)$ in $T$.

\section{Properties of $g$} \label{section:g}

In this section, we study the value function
\begin{equation} \label{eqn:val-g}
\val_g(T) \ = \ \sum_{x \le_T y} \card{T[x \lor y]} g(x,y),
\end{equation}
where we only take into account the order relation, ignoring the similarity.
The first part, Section~\ref{section:g-ideal-inputs}, is concerned with the behaviour on ideal inputs.
That is, we show that when $\omega$ is a binary partial order, then the optimal trees under $\val_g$
are order preserving.
Next, in Section~\ref{section:g-efficacy}, we turn to relaxed orders,
and present a quantitative analysis of $\val_g$,
showing that the efficacy of $\val_g$ with respect to order preservation is in the same
class as Dasgupta's model is with respect to clustering.
Finally, Section~\ref{section:migration} presents a worked example,
showing how~\eqref{eqn:val-g} can be used to analyse migration patterns between states.

\subsection{Optimality on partially ordered data} \label{section:g-ideal-inputs}
The goal of this section is to show that if $(X,\le)$ is a partially ordered set,
and if we define $\omega = \ind_\le$, then the optimal trees under $\val_g$ are order preserving.
This is important, for it shows that when the input is well formed, then order preserving
hierarchical clustering works as intended.
Therefore, in what follows, let that $\omega = \ind_\le$.

Before we embark on the main theorem, we introduce some new tools.
For a partially ordered set $(X,\le)$ and a tree $T \in \btrees{X}$, every pair of elements in $X$
are evaluated in $\val_g$ according to how they are ordered under $\le_T$. We define
\deft{the $T$-symmetrisation of $g$} to be the function
$\gamma_T : X \times X \to \R$, defined as
\[
\gamma_T(x,y) \ = \
\begin{cases}
  g(x,y) & \text{if $x \le_T y$}, \\
  g(y,x) & \text{otherwise}.
\end{cases}
\]
That is, $\gamma_T$ is a symmetric function that computes the
value under $g$ consistent with the orientation on the arguments induced by $T$.

Moreover, given a partially ordered set $(X,\le)$ and a tree $T \in \btrees{X}$,
we define \deft{the cluster graph over $(X,\le)$ and $T$} to be
the weighted undirected complete graph $G_T = (X,E,\nu)$, where the weight function is given by
\[
\nu(x,y) \ = \ \card{T[x \lor y]} \gamma_T(x,y).
\]
That is, every edge $(x,y)$ in $G_T$ has a weight that corresponds to the value of the pair
$\{x,y\}$ in~$\val_g(T)$. The $T$-symmetrisation ensures that the orientation induced by $T$
is adhered to.
The following lemma points up the purpose of the above construction; it allows us to replace the
formula for the value $\val_g$ of a tree by a sum over the edge weights of the cluster graph.
\begin{lemma} \label{lemma:cluster-graph-sum}
  If $G_T=(X,E,\nu)$ is the cluster graph over $(X,\le)$ and $T$, then
  \[
  \val_g(T) \ = \ \sum_{e \in E} \nu(e).
  \]
\end{lemma}
\begin{proof}
  Every distinct pair $\{x,y\} \in X \times X$ is summed over exactly once, and since $\gamma_T$
  ensures that the pair is associated with the value under $g$ corresponding to the orientation
  induced by $T$, the lemma holds.
\end{proof}

The next lemma allows us to manipulate the cluster graph in a controlled manner.
Recall that for a tree $T \in \btrees{X}$ with $T=(V,E)$, every node $S \in V$ is a subset of $X$,
and every edge $e \in E$ is a subset inclusion.
\begin{lemma} \label{lemma:induced-tree}
  Let $\phi : X \to X$ be a bijection, and let $T=(V,E)$ be a binary tree over $X$.
  Then there exists a binary tree $T'=(V',E')$ over $X$ that is isomorphic to $T$,
  and where the isomorphism $\phi_T : V \to V'$ is given by
  $$\phi_T(S) \ = \ \{\, \phi(x) \,|\, x \in S\,\}.$$
  In particular, for all $x,y \in X$, we have
  \begin{equation} \label{eqn:induced-cardinalities}
  \card{T[x \lor y]} \ = \ \card{T'[\phi(x) \lor \phi(y)]}.
  \end{equation}
\end{lemma}
\begin{proof}
  The set $V$ is a subset of the power set of $X$, and the map $\phi_T$ is the mapping from
  the power set of $X$ to itself induced by $\phi$, restricted to~$V$.
  Since $\phi$ is a bijection, the map $\phi_T : V \to V'$ is also a bijection.
  Moreover, since $\phi_T$ preserves subset inclusion,
  it follows that the tree structure is preserved under $\phi_T$. Thus, $T'$ is a well-definded binary tree
  over $X$, and is isomorphic to $T$.
  Equation~\eqref{eqn:induced-cardinalities} follows since
  $\phi_T$ preserves cardinalities.
\end{proof}

Given $\phi$ and $T$ from Lemma~\ref{lemma:induced-tree}, the tree $T'$ is necessarily unique.
We refer to $T'$ as \deft{the tree induced by $T$ and $\phi$}.

\medskip

We are now ready to state our theorem, telling us that
for a partially ordered set, the optimal trees under $\val_g$ are order preserving:
\begin{theorem} \label{thm:partial-orders-optimal-trees}
  If $(X,\le)$ is a partially ordered set and $T \in \btrees{X}$ is not order preserving, then there exists
  a bijection $\phi : X \to X$ so that the tree $T'$ induced by $T$ and $\phi$ is order preserving and
  has strictly higher value than $T$.
  In particular, the optimal trees over $X$ are order preserving.
\end{theorem}
\begin{proof}
  If $T$ is not order preserving, we can pick distinct $a,b \in X$ for which $a \le b$ and $b \le_T a$.
  Let $\sigma$ be the permutation on $X$ swapping $a$ and $b$, and leaving all other elements fixed,
  and let $T'$ be the tree induced by~$T$ and~$\sigma$.
  We show that $\val_g(T) < \val_g(T')$.

  Let $G_T=(X,E,\nu)$ be the cluster graph over $(X,\le)$ and $T$, and let
  $G_{T'}=(X,E,\nu')$ be the cluster graph over $(X,\le)$ and $T'$.
  We claim that both of the following holds:
  \begin{align}
    \exists e \in E \ &: \ \nu'(e) > \nu(e), \label{eqn:one-increasing-edge} \\
    \forall e \in E \ &: \ \nu'(e) \ge \nu(e). \label{eqn:no-decreasing-edges}
  \end{align}
  If so, according to Lemma~\ref{lemma:cluster-graph-sum},
  $T'$ has a strictly higher value than $T$. We may then repeat the process of
  permuting elements and generating induced trees of strictly higher value
  until we reach an order preserving tree.
  Since a sequence of permutations is a bijection, it follows that the order preserving
  tree is induced by $T$ and this sequence of permutations.

  \medskip

  To prove the claim, since no element outside the subtree $T[a \lor b]$ is affected by the permutation,
  we can assume that $a$ and $b$ are joined at the root of $T$, without loss of generality.

  \medskip

  To prove~\eqref{eqn:one-increasing-edge}, notice that
  \[
  \nu(a,b) \ = \ \card{X}\gamma_T(a,b) = -\card{X} < \card{X} = \card{X} \gamma_{T'}(a,b) = \nu_\sigma(a,b).
  \]

  To prove~\eqref{eqn:no-decreasing-edges}, we partition the edges $E$ into five disjoint sets:
  \begin{equation}  \label{eqn:E-decomp}
    \begin{aligned}
      E_1 &= \{ (x,y) \in E \, | \, x,y \ne a,b \,\} &
      E_4 &= \{ (a,x) \in E \, | \, a \le x \, \land \, b \not \le x \,\} \\
      E_2 &= \{ (x,a) \in E \, | \, x \le a \,\} &
      E_5 &= \{ (x,b) \in E \, | \, x \not \le a \, \land \, x \le b \,\} \\
      E_3 &= \{ (b,x) \in E \, | \, b \le x \,\}
    \end{aligned}
  \end{equation}

  {\it Case $E_1$:}
  Since all edges in $E_1$ are between fixed points under $\sigma$,
  there are no changes of the edge weights.

  {\it Case $E_2$:}
  Let $(x,a) \in E_2$. Then $x \le a \le y$, since $\le$ is transitive.
  From Lemma~\ref{lemma:induced-tree},
  and since $\gamma_T(x,a) = \gamma_{T'}(x,b)$ and $\gamma_T(x,b) = \gamma_{T'}(x,a)$, we get
  \[
  \setlength{\arraycolsep}{3pt}
  \begin{array}{rcccccl}
  \nu(x,a)  &=&
  \card{T[x \lor a]} \gamma_T(x,a) &=&
  \card{T'[\sigma(x) \lor \sigma(b)]} \gamma_{T'}(x,b) &=& \nu'(x,b), \\[.3em]
  \nu(x,b)  &=&
  \card{T[x \lor b]} \gamma_T(x,b) &=&
  \card{T'[\sigma(x) \lor \sigma(a)]} \gamma_{T'}(x,a) &=& \nu'(x,a). \\
  \end{array}
  \]
  Hence, the edge weights are merely swapped around by $\sigma$.
  Case $E_3$ is proven by a symmetric argument.

  {\it Case $E_4$:}
  Let $X$ split into $(A,B)$ at the root of $T$ so that $b \in A$ and $a \in B$, and let $(a,x) \in E_4$.
  Now, if $x \in A$, then $\nu(a,x) = -\card{X}$, which is the lowest possible value, so
  $\nu'(a,x) \ge \nu(a,x)$. And if $x \in B$, then $\nu'(x,a) = \card{X}$,
  which is the largest possible value, so also in this case, $\nu'(a,x) \ge \nu(a,x)$.

  A symmetric argument covers case $E_5$.
\end{proof}

Recalling the random family of graphs over a partially ordered set given in Example~\ref{ex:random-family},
the following corollary states that in expectation, the trees of maximal value are order preserving.
\begin{corollary}
  Let $(X,\le)$ be a partially ordered set, and let $G=(V,E)$ be a graph according to the random family
  of Example~\ref{ex:random-family}. If we define $\omega = \ind_E$, and if $T$ is a non-order preserving
  tree with respect to $\le$, then there exists an order preserving tree $T' \in \btrees{X}$ with
  $\E\! \val_g(T') > \E\! \val_g(T)$.
\end{corollary}
\begin{proof}
  First, for $\alpha > 0$, because $\val_{\alpha g} = \alpha \val_g$, the optimal trees under
  $\val_{\alpha g}$ coincide with the optimal trees under $\val_g$.
  Hence, optimisation under $\val_g$ is invariant to positive scaling of $g$.

  Second, given $(X,\omega)$, the optimal trees under $\val_g$ are invariant with respect
  to a large class of affine transformations of $\omega$. This is because
  if $\alpha,\beta \in \R$ with $\alpha > 0$,
  and if we define $\omega ' = \alpha \omega + \beta$, we get
  \[
  g'(x,y) \, = \, \omega'(x,y) - \omega'(y,x) \, = \, \alpha \omega(x,y) + \beta - \alpha \omega(y,x) - \beta
  \, = \, \alpha g(x,y).
  \]

  Now, given $\omega$ from Example~\ref{ex:random-family}, and defining
  $\omega' = \tfrac{1}{p-q}\left(\omega - q \right)$, we get $\omega' = \ind_\le$.
  That is, the optimal trees under $\val_g$ for $\omega=\ind_\le$ coincide with the optimal
  trees under $\val_g$ for $\omega$ as given in Example~\ref{ex:random-family}.
  Since the latter is the expected value of $\omega$ for a graph from the random family,
  the corollary follows.
\end{proof}

\subsection{Efficacy on planted partial orders} \label{section:g-efficacy}
In this section, we provide quantitative results on the difference in order preservation between
optimal trees and suboptimal trees for a simple class of partial orders.
We show that our model partitions partially ordered input with power comparable to how Dasgupta's model
partitions cliques.

\medskip

For the duration of this section, we fix $\card{X}=n$ even.




\subsubsection{Bounding the number of reversed pairs for planted bipartite partial orders}
\label{section:g-quantitative}
We start by defining our notion of a planted partial order.
\begin{definition} \label{def:bpp}
  Let $X$ be a set, $(A^*,B^*)$ a split of $X$ into equally large blocks, and $p,q$ real numbers
  for which $0 \le q < p \le 1$. Let $\le^*$ denote the smallest partial order on $X$
  satisfying $(x,y) \in A^* \times B^* \Rightarrow x \le^* y$, and let $\Gamma$ denote the
  stochastic family of directed graphs $G=(X,E)$ for which
  \[
  \Pr\!\big((x,y) \in E\big) =
  \begin{cases}
    p & \text{if $x \le^* y$}, \\
    q & \text{otherwise}.
  \end{cases}
  \]
  We refer to the partially ordered set $(X,\le^*)$ as
  \deft{the planted bipartite partial order defined by $X$, $A^*$, $B^*$, $p$ and $q$}.
\end{definition}

Notice that if $G=(X,E)$ is drawn from $\Gamma$, then $G$ gives rise to a relaxed order on $X$ defined
by $\omega_G = \ind_\le$.
In particular, we have $\E[\omega_G(x,y)] = \Pr\!\big((x,y) \in E\big)$,
so that if we have $g(x,y) = \omega_G(x,y)-\omega_G(y,x)$, we get
\begin{equation} \label{eqn:Eg}
\E\! g(x,y) =
\begin{cases}
  p - q & \text{if $x <^* y$}, \\
  q - p & \text{if $y <^* x$}, \\
  0 & \text{otherwise}.
\end{cases}
\end{equation}
As a consequence, any tree that splits into $(A^*,B^*)$ at the root is of maximal expected value.


\bigskip

The next definition serves to define the quality of a binary tree in terms of
order preservation. For a partially ordered set $(X,\le$), we write
$\pairs{<}$ to denote the number of comparable pairs of distinct elements under $\le$; that is
$\pairs{<} = \card{\{(x,y) \in X \times X \,|\, x < y\}}$.
\begin{definition} \label{def:delta-good}
  Let $(X,\le)$ be a partial order, and let $\delta \in \left[0,\half\right]$.
  A tree $T \in \btrees{X}$ is \deft{$\delta$-good (with respect to $\le$)} if,
  for every split $(A,B)$ in $T$, we have
  \begin{equation*}
    \frac{\ind_\le(B,A)}{\pairs{<}} \le \delta.
  \end{equation*}
  That is, the fraction of reversed pairs over the split against the total
  number of comparable pairs in $(X,\le)$ is no larger than $\delta$.
\end{definition}
Clearly, a tree is order preserving if and only if it is $0$-good. Also, for the bipartite planted
partial order, any tree that splits into $(A^*,B^*)$ at the root is $0$-good with respect to $\le^*$.

\bigskip

The first result of this section shows that if we define our value function based on a random graph
from $\Gamma$, then there is a significant difference in
expected value between an optimal tree and one that is not $\delta$-good.
\begin{lemma} \label{lemma:not-delta-good}
  Let $G=(X,E)$ be drawn from $\Gamma$, and define $g(x,y)=\ind_E(x,y)-\ind_E(y,x)$.
  If $T^*$ is a tree that splits $X$ into $(A^*,B^*)$ at the root,
  and if $T$ is a tree that is not $\delta$-good with respect to $\le^*$, then
 \[
  \E\! \val_g(T^*) > \E\! \val_g(T) + (n+2) (p-q) \delta \frac{n^2}{4}.
  \]
\end{lemma}
\begin{proof}
  Since $T$ is not $\delta$-good, there is a split $(A,B)$ in $T$ where
  $\ind_{\le^*}(B,A)/\pairs{<^*} > \delta$.
  By construction, we have $\pairs{<^*} = \card{A^*}\card{B^*} = \frac{n^2}{4}$, and since the split
  $(A,B)$ involves reversed pairs, we must have $\card{A \cup B} \ge 2$.
  Hence, the reversed pairs contribute negatively in the split with magnitude
  \[
  (p-q) \card{A \cup B} \ind_{\le^*}(B,A) >
  \card{A \cup B} (p-q) \delta \pairs{<^*} \ge 2 (p-q) \delta \pairs{<^*} = 2 (p-q) \delta \frac{n^2}{4}.
  \]

  Moreover, the pairs contributing to $\ind_{\le^*}(B,A)$ fail to contribute positively to the value
  of the root split by a magnitude
  \[
  \card{X} (p-q) \ind_{\le^*}(B,A) \ge n (p-q) \delta \pairs{<^*} = n (p-q) \delta \frac{n^2}{4}.
  \]
  Adding the magnitudes yields the statement of the lemma.
\end{proof}

The above result is on the expected value $\E\! \val_g(T)$. The next result bounds the difference
between $\E\! \val_g(T)$ and $\val_g(T)$.

\begin{lemma} \label{lemma:McDiarmid}
  Let $G=(X,E)$ be a random graph from $\Gamma$, let
  $g(x,y) = \ind_E(x,y) - \ind_E(y,x)$, and pick $\vareps \in (0,1)$.
  Then we have, for any $T \in \btrees{X}$, with probability $1-\vareps$,
  \[
  \abs{\val_g(T) - \E\! \val_g(T)}
  <
  n^2 \sqrt{2n \ln 2n + \ln \frac{2}{\vareps}}.
  \]
\end{lemma}
\begin{proof}
  The proof is identical to that of \citep[Lemma 8]{Dasgupta2016} with one modification:
  If, for every pair $x <_T y$ of elements in $X$, we consider the functions $g(x,y)$
  to be separate independent random variables,
  then the magnitude of change of $\val_g(T)$ is bounded by
  $2n$ whenever only one of these random variables are allowed to change. This because $g \in [-1,1]$.
  The result now follows from an application of McDiarmid's inequality and Dasgupta's proof.
\end{proof}

We now combine the two above lemmas, showing that for a random graph drawn from the family
corresponding to the bipartite planted partial order, that a tree that is optimal with respect
to that graph has an upper bound
on the fraction of reversed pairs given by $O\big(\sqrt{(\log n)/n}\big)$.

\begin{theorem} \label{thm:splitting-power}
  Let $G=(X,E)$ be a random graph from $\Gamma$, let $g(x,y) = \ind_E(x,y) - \ind_E(y,x)$, and
  pick $\vareps \in (0,1)$.
  If $T$ is an optimal tree for $\val_g$, then $T$ is $\delta$-good with respect to $\le^*$ with
  probability $1-\vareps$ for
  \[
  \delta = \frac{8}{p-q} \sqrt{
    \frac{2 \ln 2n}{n} + \frac{1}{n^2} \ln \frac{2}{\vareps}
  }.
  \]
\end{theorem}
\begin{proof}
  Let $T^*$ be a tree that splits into $(A^*,B^*)$ at the root.
  Since $\val_g(T) \ge \val_g(T^*)$, we can deduce that
  \begin{align*}
    \E\! \val_g(T^*) - \E\! \val_g(T)
    &=   \E\! \val_g(T^*) - \val_g(T^*) + \val_g(T^*) - \E\! \val_g(T) \\
    &\le \E\! \val_g(T^*) - \val_g(T^*) + \val_g(T) - \E\! \val_g(T) \\
    &\le \abs{\E\! \val_g(T^*) - \val_g(T^*)} + \abs{\val_g(T) - \E\! \val_g(T)} \\
    &\le 2 n^2 \sqrt{2n \ln 2n + \tfrac{2}{\vareps}}
    \intertext{with probability $1-\vareps$. Hence, we have}
    \E\! \val_g(T^*) &\le \E\! \val_g(T) + 2 n^2 \sqrt{2n \ln 2n + \tfrac{2}{\vareps}}
  \end{align*}
  with probability $1-\vareps$ too.

  \medskip

  Now, if $T$ is not $\delta$-good, Lemma~\ref{lemma:not-delta-good} yields
  \[
  \E\! \val_g(T) + (n+2) (p-q) \delta \pairs{<^*}
  \le
   \E\! \val_g(T) + 2 n^2 \sqrt{2n \ln 2n + \tfrac{2}{\vareps}},
  \]
  giving us
  \[
  \delta
  \le
  \frac{
    2 n^2 \sqrt{2n \ln 2n + \tfrac{2}{\vareps}}
  }{
    (n+2) (p-q) \, \pairs{<^*}
  }
  <
  \frac{
    2 n^2 \sqrt{2n \ln 2n + \tfrac{2}{\vareps}}
  }{
    (p-q) \tfrac{n^3}{4}
  }
  =
  \frac{8}{p-q} \sqrt{
    \frac{2 \ln 2n}{n} + \frac{1}{n^2} \ln \frac{2}{\vareps}
  }.
  \]
\end{proof}

\paragraph*{On the number of comparable pairs in the bipartite planted partial order.}
By construction, the number of comparable pairs between $A^*$ and $B^*$ in the planted partial order
is $\frac{n^2}{4}$, which is as high as it can be. An interesting question is to how many pairs we need
in order to be able to separate the two blocks.
Re-tracing the steps in the above calculations while letting $\pairs{<^*}$
denote the number of comparable pairs between the blocks, gives us a fraction of reversed
pairs of $O\big((n^2/\pairs{<^*})\sqrt{\log(n)/n}\big)$.
Indeed, if we choose to have as few comparable pairs as possible while still separating the two blocks,
we can manage with $\pairs{<^*}=\frac{n}{2}$.
This yields an asymptotic fraction of $O\big(\sqrt{n \log n}\big)$ reversed pairs.

For the fraction of reversed pairs to diminish as $n \to \infty$,
we need $\pairs{<^*}= \Omega(n^2)$\footnote{Where $\Omega$ is according to \citet{KnuthBigOmicron}}.
This can be achieved, for example, by having all points in a fixed fraction $\alpha$ of the
elements of $A^*$ being related to all points of an equally large fraction of $B^*$,
while the remaining pairs of points are related to exactly one point in the other set.
This yields $$\pairs{<^*}=\alpha^2 n^2/4 + (1-\alpha)n/2 = \Omega(n^2)$$ for any choice
of $\alpha > 0$.

\subsection{A worked example -- migration between states} \label{section:migration}
This section presents a worked example, analysing migration flow between states in the USA.
The purpose of the example is to illustrate how order preserving hierarchical
clustering with $\val_g$ behaves on data that is not partially ordered, but has a directed nature.
As the example shows, the produced hierarchical clustering is similar to what can be obtained by
applying other methods already in use for hierarchical clustering, such as \MaxDiCut{} \citep{FeigeGoemans1995}
and \DirSparsestCut{} \citep{ChuzhoyKhanna2006}.
However, for these methods, it is not obvious how to combine them with a similarity.

The states being subject to this analysis, and the migration data, is presented in
Table~\ref{table:migration-data}. The values are obtained from \url{http://www.census.gov}, and represent
migration data between US states in the year $2011$.

\begin{table}[htpb]
  \begin{center}
    \begingroup
    \small
    \begin{tabular}{r@{$:\ $}l@{$\quad$}l|ccccccc}
      \multicolumn{3}{c|}{} & 1 & 2 & 3 & 4 & 5 & 6 & 7 \\
      \cline{3-10}
      1 & Arizona & $\phantom{,}$1 & $-$ & $0.064$ & $0.006$ & $0.018$ & $0.014$ & $0.012$ & $0.022$ \\
      2 & California & $\phantom{,}$2 & $0.089$ & $-$ & $0.016$ & $0.072$ & $0.061$ & $0.033$ & $0.069$ \\
      3 & Idaho & $\phantom{,}$3 & $0.004$ & $0.009$ & $-$ & $0.007$ & $0.011$ & $0.014$ & $0.020$ \\
      4 & Nevada & $\phantom{,}$4 & $0.016$ & $0.065$ & $0.006$ & $-$ & $0.013$ & $0.008$ & $0.009$ \\
      5 & Oregon & $\phantom{,}$5 & $0.008$ & $0.033$ & $0.013$ & $0.003$ & $-$ & $0.004$ & $0.052$ \\
      6 & Utah & $\phantom{,}$6 & $0.019$ & $0.016$ & $0.011$ & $0.006$ & $0.006$ & $-$ & $0.009$ \\
      7 & Washington & $\phantom{,}$7 & $0.025$ & $0.065$ & $0.016$ & $0.008$ & $0.039$ & $0.009$ & $-$ \\
    \end{tabular}
    \endgroup
    \caption{Migration data normalised so that the total migration sums to one. The table displays the
    migration magnitudes in ``from row to column''-fashion.}
    \label{table:migration-data}
  \end{center}
\end{table}

\begin{figure}[htpb]
  \begin{center}
    \begin{tabular}{cc}
      \scalebox{0.7}{
        \newcommand{\brk}[1]{
          \begingroup
          \setlength{\arraycolsep}{1pt}
          \boldmath{$\left\{ \begin{array}{c} #1 \end{array} \right\}$}
          \endgroup
        }
        \begin{tikzpicture}[yscale=1.8,xscale=1.3]
          \node (n0) at (2.5,6) {\brk{\mathrm{Az,Ca,Id,Nv} \\ \mathrm{Or,Ut,Wa}}};
          \node (n11) at (1,5) {\brk{\mathrm{Ca}}};
          \node (n12) at (4,5) {\brk{\mathrm{Az,Id,Nv} \\ \mathrm{Or,Ut,Wa}}};
          \node (n21) at (2.5,4) {\brk{\mathrm{Az,Id,Nv} \\ \mathrm{Or,Ut}}};
          \node (n22) at (5.5,4) {\brk{\mathrm{Wa}}};
          \node (n31) at (1,3) {\brk{\mathrm{Az,Nv,Ut}}};
          \node (n32) at (4,3) {\brk{\mathrm{Id,Or}}};
          \node (n41) at (0.25,2) {\brk{\mathrm{Ut}}};
          \node (n42) at (1.75,2) {\brk{\mathrm{Az,Nv}}};
          \node (n43) at (3.25,2) {\brk{\mathrm{Or}}};
          \node (n44) at (4.75,2) {\brk{\mathrm{Id}}};
          \node (n51) at (1,1) {\brk{\mathrm{Az}}};
          \node (n52) at (2.5,1) {\brk{\mathrm{Nv}}};

          \draw [line width=1.5pt] (n0)  -- (n11);
          \draw [line width=1.5pt] (n0)  -- (n12) -- (n22);
          \draw [line width=1.5pt] (n12) -- (n21) -- (n31) -- (n41);
          \draw [line width=1.5pt] (n21) -- (n32) -- (n43);
          \draw [line width=1.5pt] (n32) -- (n44);
          \draw [line width=1.5pt] (n31) -- (n42) -- (n51);
          \draw [line width=1.5pt] (n42) -- (n52);
        \end{tikzpicture}
      }
      &
      \scalebox{0.5}{
        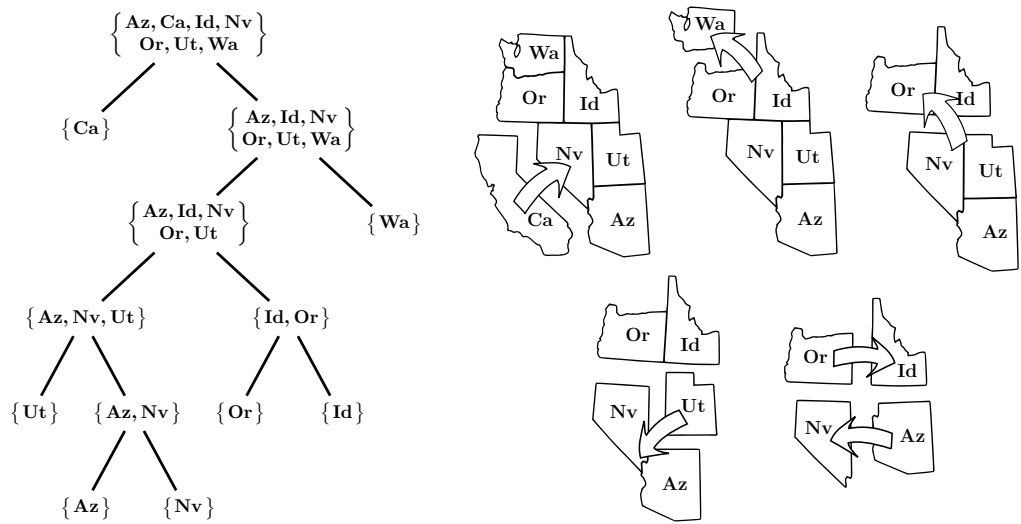
      }
    \end{tabular}
    \caption{Figure showing the result of the clustering of the migration data. The binary tree is
      displayed to the left, and the sequence of splits of the states are shown to the right. The arrows
      on the splits indicate the direction of net migration.}
    \label{fig:migration}
  \end{center}
\end{figure}

We let the function $\omega$ represent the flow of people moving from one state to another in the course of one
year. The hierarchical clustering proceeds by splitting the set of states in two, providing the two blocks
of states having maximal net flow of migrants from one block to the other. Then these blocks will be split
accordingly, and so on, until only single states remain.

The result of the hierarchical clustering of this data is presented in~Figure~\ref{fig:migration}.
We see that the first state to be split off is California, indicating that
the largest flow of migration is \emph{from} California \emph{to} all the other states.
Second, Washington is split
off from the remaining states, and the orientation (arrow) tells us that the net flow of migration is from
Arizona, Idaho, Oregon, Nevada and Utah to Washington. Third, Oregon and Idaho is split off from
Nevada, Utah and Arizona, with net migration from the latter group to the former, and so on.
The linear order on the states induced by the binary tree is
\[
\mathrm{Ca} \ \le_T \ \mathrm{Ut} \ \le_T \ \mathrm{Az} \ \le_T \ \mathrm{Nv} \ \le_T \ \mathrm{Or} \ \le_T \
\mathrm{Id} \ \le_T \ \mathrm{Wa},
\]
indicating the general direction of net migration flow.

\section{Properties of $f=s_d+g$} \label{section:f}

We now turn to study the full objective function $f$ as defined in~\eqref{eqn:f}; that is, the
sum of the dissimilarity $s_d$ and the antisymmetrisation $g$ of the relaxed order
relation.

We start by looking at an example, showing how we can use order preserving hierarchical
clustering to recover the ancestral tree of former U.S.\ president John F.\ Kennedy.
The example illustrates how the dissimilarity and the order relation combine to
successfully recover the ancestral tree in a situation where neither objective
could have managed alone.

The example also shows us that the two objectives must be balanced in order to
achieve both good clustering and good order preservation. This is the topic of the second
part of the section, where we consider the problem of finding the Pareto optimal trees in the context
of bi-objective optimisation.
We show that due to the linearity of the objectives, we can recover the entire Pareto
front by studying the convex combinations of $s_d$ and $g$.

\subsection{Analysing the JFK ancestral tree}
\label{section:JFK}
We will now illustrate how the dissimilarity and the order relation combine to successfully
recover the  ancestral tree of John F. Kennedy in a situation where neither of the objectives can
do this alone.
The idea is to cluster the individuals in the tree so that closely related individuals are grouped
together, while at the same time keeping the generations apart. The dissimilarity is derived from
name differences, and the order relation is based on the ancestor-descendant relations among
the individuals.
The ancestral tree is depicted in Figure~\ref{fig:kennedys}, and the dissimilarities are
listed in Table~\ref{tab:Kennedy-dissimilarities}.

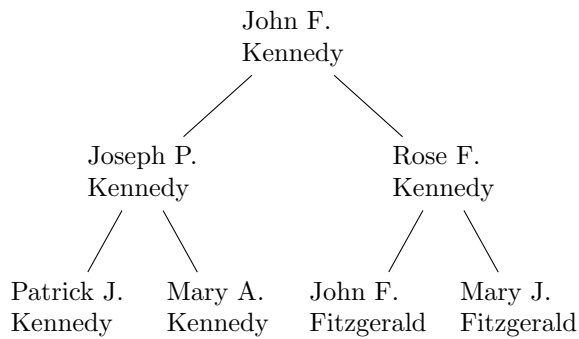
\begin{figure}[tphp]
  \begin{center}
    \begin{tikzpicture}[yscale=1.8]
      \node (n00) at (5,5) {\pbox[c]{2cm}{John F.\\ Kennedy}};
      \node (n10) at (3,4) {\pbox{4cm}{Joseph P.\\ Kennedy}};
      \node (n11) at (7,4) {\pbox{4cm}{Rose F.\\ Kennedy}};
      \node (n20) at (2,3) {\pbox{4cm}{Patrick J.\\ Kennedy}};
      \node (n21) at (4,3) {\pbox{4cm}{Mary A.\\ Kennedy}};
      \node (n22) at (6,3) {\pbox{4cm}{John F.\\ Fitzgerald}};
      \node (n23) at (8,3) {\pbox{4cm}{Mary J.\\ Fitzgerald}};
      \draw (n00) -- (n10) -- (n20);
      \draw (n10) -- (n21);
      \draw (n00) -- (n11) -- (n22);
      \draw (n11) -- (n23);
    \end{tikzpicture}
  \end{center}
  \caption{John F. Kennedy's ancestral tree, two generations back.}
  \label{fig:kennedys}
\end{figure}

\begin{table}[htpb]
  \begin{center}
    \begin{tabular}{c@{\hskip 1cm}c}
      \begingroup
      \setlength{\tabcolsep}{1pt}
      \begin{tabular}{r@{$:\ $}lll}
        1 & John & F. & Kennedy \\
        2 & Joseph &P.& Kennedy \\
        3 & Rose &F.& Kennedy \\
        4 & Patrick &J.& Kennedy \\
        5 & Mary &A.& Kennedy \\
        6 & John &F.& Fitzgerald \\
        7 & Mary &J.& Fitzgerald
      \end{tabular}
      \endgroup
      &
      $%
      \begin{array}{c|cccccc}
        & 2 & 3 & 4 & 5 & 6 & 7 \\
        \hline
        1 & 0.29 & 0.31 & 0.53 & 0.53 & 0.50 & 0.63  \\
        2 && 0.40 & 0.50 & 0.59 & 0.62 & 0.73  \\
        3 &&& 0.61 & 0.53 & 0.65 & 0.70  \\
        4 &&&& 0.44 & 0.50 & 0.47  \\
        5 &&&&& 0.65 & 0.56  \\
        6 &&&&&& 0.28  \\
      \end{array}
      $%
    \end{tabular}
    \caption{Dissimilarities of names in the ancestral tree.
      On the left are mappings from individuals to indices, and on the right
      is the table of name dissimilarities, rounded to two decimal places. The dissimilarities
      correspond to the Jaccard distances between the names.}
    \label{tab:Kennedy-dissimilarities}
  \end{center}
\end{table}

Notice that, in the ancestral tree, the Kennedy name is present throughout, and in particular across
the generations, yielding high similarities between descendant- and ancestor names.
Due to this, attempting to cluster using only the dissimilarity will cause descendants and
ancestors to be placed in the same cluster.

On the other hand, defining $\omega$ so that $\omega(x,y)=1$ if and only if $x$ is a descendant of
$y$ (and zero otherwise), this will keep generations apart, but $w$ holds no information about
how to identify groups of individuals within the same generation.

If we combine the two objectives as in $f = s_d + g$,
the corresponding optimal hierarchical clustering is presented in Figure~\ref{fig:Kennedy-results}.
We see that the generations are nicely split apart, and the grandparents are nicely grouped together,
as we wished for.
Hence, by combining the dissimilarity and the order relation, the result turns out correctly.

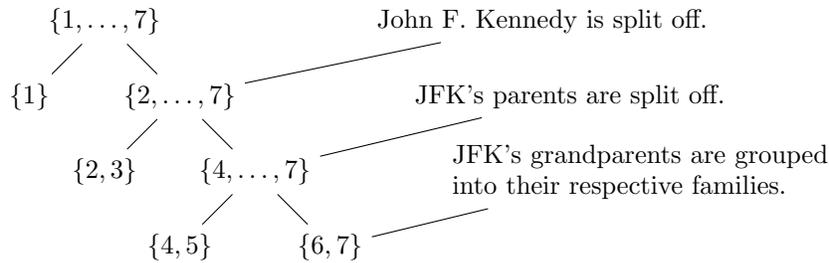
\begin{figure}[htpb]
  \begin{center}
    \begin{tikzpicture}
      \node (n00) at (2,4) {$\{1,\ldots,7\}$};
      \node (n10) at (1,3) {$\{1\}$};
      \node (n11) at (3,3) {$\{2,\ldots,7\}$};
      \node (n20) at (2,2) {$\{2,3\}$};
      \node (n21) at (4,2) {$\{4,\ldots,7\}$};
      \node (n30) at (3,1) {$\{4,5\}$};
      \node (n31) at (5,1) {$\{6,7\}$};
      \draw (n00) -- (n11) -- (n21) -- (n31);
      \draw (n00) -- (n10);
      \draw (n11) -- (n20);
      \draw (n21) -- (n30);
      \node[anchor=west] (s1) at (5.5,4) {\pbox{7cm}{John F. Kennedy is split off.}};
      \draw (s1) -- (n11);
      \node[anchor=west] (s2) at (6.0,3) {\pbox{7cm}{JFK's parents are split off.}};
      \draw (s2) -- (n21);
      \node[anchor=west] (s3) at (6.5,2) {\pbox{5cm}{JFK's grandparents are grouped into
          their respective families.}};
      \draw (s3) -- (n31);
    \end{tikzpicture}
    \caption{The optimal hierarchical clustering of the Kennedy family tree. We have left
      out the final splits into leaf nodes.}
    \label{fig:Kennedy-results}
  \end{center}
\end{figure}

Looking back at our motivating use case (Section~\ref{section:motivation}), the database of machinery
exhibits the exact same tendency: The high similarity between parts and sub-parts makes clustering difficult,
but we can mitigate this problem by taking the part-of relation into account.

\subsection{Balancing clustering against order preservation}
In the above example, the dissimilarity $s_d$ and the order relation $\omega$
are, in a sense, competing: the dissimilarity wants to place similar elements together, in spite
of them being in different generations, and the order relation tries to keep the generations apart.
In the field of operations research, this is termed a \emph{bi-objective optimisation problem},
having two distinct objectives where ideally both shall be optimised for.
Multi-objective optimisation is a thoroughly
studied area of research, and several methods exists to approach this class of
problems. See, for example, the survey by \citet{MarlerArora2004} for an overview.

The material presented in this section are known results within multiobjective optimisation.
We still choose to write it out, for the sake of completeness and also because it reveals
an efficient approach to explore different optimal solutions.

We have chosen to focus on \emph{Pareto optimality} \citep{Zadeh1963} in the context
of bi-objective maximisation. In order to do that, for an ordered similarity space $(X,s,\omega)$,
we decompose $\val_f$ into its sum components $\val_{s_d}$ and $\val_g$:
\begin{definition} \label{def:pareto}
  Let $(X,s,\omega)$ be an ordered similarity space. For trees $T,T' \in \btrees{X}$ we say that
  $T$ is \deft{Pareto dominated} by $T'$ if
  \begin{description}
    \item[$1.$] $\val_\theta(T') \ge \val_\theta(T)$ for $\theta \in \{s_d,g\}$, and
    \item[$2.$] for at least one objective $\val_\theta$ we have $\val_\theta(T') > \val_\theta(T)$.
  \end{description}
  A solution is \deft{Pareto optimal} if there are no solutions dominating it, and the family of all
  Pareto optimal solutions is referred to as the \deft{Pareto front}.
\end{definition}
Thus, a Pareto optimal solution has the property that for any other candidate solution,
at least one objective will deteriorate.
An illustration of the Pareto front is given in Figure~\ref{fig:pareto-front}.
\begin{figure}[htpb]
  \begin{center}
    \newcommand{\paretopt}[2][nx]{\node (#1) at #2 [square,draw] {};}
    \definecolor{shadingcol}{rgb}{0.5, 0.5, 0.5}
    \begin{tikzpicture}[scale=0.6,
        square/.style={regular polygon,
          regular polygon sides=4,minimum size=4pt,inner sep=0pt},]
      \draw[<->] (0.3,7.5) node (yaxis) [above] {$\val_g$} |- (10,0.5) node (xaxis) [right] {$\val_{s_d}$};
      \paretopt{(5,5)} \paretopt{(2,5.5)} \paretopt{(3,4)}   \paretopt{(2,1)} \paretopt{(5,1.7)}
      \paretopt{(1,3)} \paretopt{(4.8,4.6)}   \paretopt{(7,2.7)} \paretopt{(4,3.5)} \paretopt{(5.1,2)}
      \paretopt[p1]{(1,7)} \paretopt[p2]{(3.1,6.2)} \paretopt[p3]{(6,6)}
      \paretopt[p4]{(6.5,4) [fill=black]}
      \paretopt[p5]{(8,2)} \paretopt[p6]{(9,1)};
      \draw (p1) -- (p2) -- (p3) -- (p4) -- (p5) -- (p6);
      \draw[dashed] (p1) -- (p3) -- (p6);
      \path[
        shade,
        left color=black,
        right color=white,
        opacity=.5,
        shading angle=135,
        middle color=white,
      ]
      (p4) -- ++(3.5,0) -- ++(-3.5,3.5) -- cycle;
    \end{tikzpicture}
    \caption{The Pareto front of a bi-objective optimisation problem with objectives $(\val_{s_d},\val_g)$.
      The squares are the different candidate solutions, and the solid line connects the candidates at the
      Pareto front. The shaded wedge illustrates why the black square belongs on the Pareto front: there are no
      solutions above or to the right of this point. The dashed line indicates the convex hull of the
      Pareto front.}
    \label{fig:pareto-front}
  \end{center}
\end{figure}
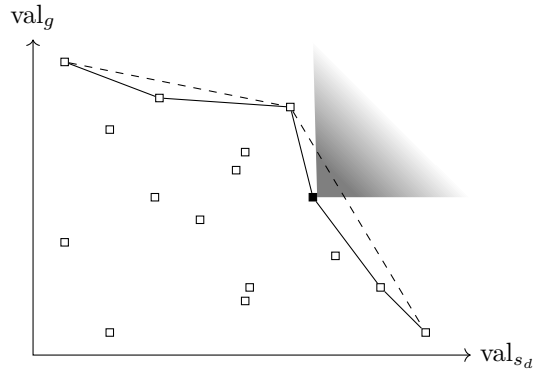

To identify the Pareto front for a bi-objective optimisation problem,
the naive approach is to identify the optimal solutions
for all linear combinations $\gamma \val_{s_d} + \delta \val_g$ for
$\gamma,\delta > 0$.\footnote{More efficient methods exist, such as \citep{KimWeck2005}.}
However, as the following definition and lemma shows, in our case, we can limit the study to
convex combinations of the objectives. We start by introducing a new value
function $\val_\alpha$, which is a convex combination of the two objectives.

\begin{definition} \label{def:val-alpha}
  Given $(X,s,\omega)$ and $\alpha \in [0,1]$, define $\val_\alpha : \btrees{X} \to \R$ as
\begin{equation} \label{eqn:val-alpha}
\val_\alpha(T) \ = \ \sum_{x \le_T y} \card{T[x \lor y]} \big[ \alpha s_d(x,y) + (1-\alpha) g(x,y) \big].
\end{equation}
\end{definition}
Notice that due to the linearity of $\val_{s_d}$ and $\val_g$, we have
\[
\val_\alpha(T) \ = \ \alpha \!\, \val_{s_d}(T) + (1-\alpha) \val_g(T).
\]

Optimising a linear combination $\gamma \val_{s_d} + \delta \val_g$ can be replaced by
optimising $\val_\alpha$ for a suitable value of $\alpha$:
\begin{lemma} \label{lemma:convex-combos}
  For $\gamma,\delta > 0$, let $\alpha = \frac{\gamma}{\gamma + \delta}$.
  Then $T \in \btrees{X}$ maximises $\val_\alpha$ if and only if $T$ also maximises
  $\gamma \val_{s_d} + \delta \val_g$.
\end{lemma}
\begin{proof}
  For a fixed $\beta > 0$, a tree that maximises $\val_{s_d} + \val_g$ also maximises
  $\beta (\val_{s_d} + \val_g)$.
  And since
  \begin{gather*}
    \frac{1}{\gamma + \delta}\big( \gamma \val_{s_d} + \delta \val_g \big)
    \ = \
    \frac{\gamma}{\gamma + \delta} \val_{s_d} + \frac{\delta}{\gamma + \delta} \val_g
    \ = \
    \alpha \val_{s_d} + (1-\alpha) \val_g,
  \end{gather*}
  and since $0 \le \frac{\gamma}{\gamma + \delta}, \frac{\delta}{\gamma + \delta} \le 1$
  and $\frac{\gamma}{\gamma + \delta} + \frac{\delta}{\gamma + \delta} = 1$, the lemma holds.
\end{proof}

According to the above Lemma, the Pareto front is convex in the sense that all Pareto optimal
solutions are located on the convex hull indicated by the dashed line in Figure~\ref{fig:pareto-front}.
Also, this means that if $0 \le \beta \le \beta' \le 1$ and $\val_\beta$ and $\val_{\beta'}$ have
the same optimal trees, then optimisation of $\val_\alpha$ is constant for $\alpha \in [\beta,\beta']$.
This is useful, since it gives a well defined basis on which to apply, for example, binary search
to explore the Pareto front.

\bigskip

The objective $\val_\alpha$ has the property that $\val_{\alpha=0} = \val_g$, optimising only with
respect to the order relation, and $\val_{\alpha=1} = \val_{s_d}$, optimising only with respect to
the similarity. Part of the optimisation problem is thus to find an $\alpha \in [0,1]$ providing
the best balance between the objectives given the current context.
For multi-objective optimisation in general, there is no canonical best solution on the Pareto front,
and therefore no canonically best $\alpha$ for $\val_\alpha$.
A significant amount of research has focused on the topic of aiding the domain expert in identifying the
best solution; see Miettinen's book \citep[pp.~131--213]{Miettinen1998} for an overview.

\bigskip

While we believe it is possible to produce a theorem along the lines of Theorem~\ref{thm:splitting-power}
also in the presence of both a similarity and an order relation,
we consider this to be a major task, and postpone this for future research.

\subsection{A note on idempotency}
A function $h : X \to X$ is \deft{idempotent} if $h \circ h = h$.
\citet{JardineSibson1971} define a clustering method as \emph{appropriate}
if it is idempotent. This is considered to be a key feature of any clustering methodology.
\citet{CohenAddadEtAl2019} show that Dasgupta's model is appropriate, and the below theorem
shows that this still holds for order preserving hierarchical clustering. An important difference is,
of course, that we must pass along the induced order relation for the second invocation.
Notice that for a normalised ultrametric $u_T$ on $X$, the function $1-u_T$ is a similarity on $X$.

\begin{theorem}
  Assume that $(X,s,\omega)$ has $T$ as an optimal tree under $\val_f$, and let $u_T$ be the normalised
  ultrametric corresponding to $T$. Then $(X,1-u_T,\le_T)$ has $T$ as an optimal tree under $\val_f$.
  That is; $\val_f$ is appropriate.
\end{theorem}
\begin{proof}
  Since $\val_{s_d}$ is appropriate, clustering the similarity space $(X,1-u_T)$ without
  any order relation will yield a binary tree $T'$ that reproduces $u_T$ under $\val_{u_T}$.
  This tree may not be order preserving, but it is isomorphic to $T$, and this tree isomorphism must
  necessarily be a pairwise swapping of elements so that the isomorphism is induced by $T'$ and these
  swaps. Hence, according to Theorem~\ref{thm:partial-orders-optimal-trees}, $T$ has at least the
  value of $T'$ under $\val_f$.
\end{proof}

The following corollary expresses the same result from a different angle: if we do not
care about the induced order relation but only about the ultrametric, we can omit the order relation
all together, given that we have a suitable similarity.

\begin{corollary}
For every ordered similarity space $(X,s,\omega)$, if $T$ is an optimal tree with respect to $\val_f$,
then there exists a similarity $s'$ on $X$ for which $T'$ is an optimal tree of $\val_{s'_d}$
and where $u_{T'} = u_T$. 
\end{corollary}

\section{Approximation} \label{section:approx}
Since optimisation of $\val_{s_d}$ is NP-hard, optimisation of $\val_f$ is NP-hard too.
In this section, we present a polynomial time approximation algorithm with a relative performance guarantee
of \bound.
The method is based on successive applications of \DirSparsestCut{} \citep{ChuzhoyKhanna2006} to produce
an order preserving tree. The method, thus, falls in the category of divisive hierarchical clustering.
The \DirSparsestCut{} is similar to the more common \SparsestCut{} \citep{LeightonRao1999}, with the
obvious difference that in \DirSparsestCut, one attempts to have as many arcs as possible pointing
in the same direction across the cut.

Notice that graph cuts, and what we have called splits, are the same thing.
In a split, we focus on \emph{splitting a set} in two, while for a graph cut, we focus on
\emph{cutting the arcs spanned by the split}. However, since every relaxed order relation
$\omega : X \times X \to [0,1]$ corresponds to a complete directed weighted graph,
the concepts coincide.

\begin{definition}
  Given a complete directed weighted graph over $X$ with weight function $\nu$,
  the \deft{directed cut density} of a split $(A,B)$ of $X$ is the magnitude
  \[
  \frac{\nu(A,B)}{\card{A} \card{B}} \ = \ \frac{\sum_{a,b} \nu(a,b)}{\card{A}\card{B}}.
  \]
  In particular, a \deft{directed sparsest cut} is a split of minimal cut density,
  taken over all possible binary splits $(A,B)$ of $X$.
\end{definition}

Notice that, just as for splits in trees, a directed cut is oriented, in the sense that $(A,B)$ may
have a different cut density than $(B,A)$.

\medskip

Also notice that, for an undirected graph, the weight function is symmetric, so that the cut densities
of $(A,B)$ and $(B,A)$ coincide. In this case, we drop the prefix, and refer
to $(A,B)$ as a possibly \deft{sparsest cut}.

\subsection{Duality}
To use \DirSparsestCut{} for our problem, we must solve it as a minimisation problem.
We therefore introduce the dual to $\val_f$, namely $\cost_{f_d}$, and show that maximisation
under $\val_f$ is equivalent to minimisation under $\cost_{f_d}$. We start by showing that there is
a dual function to $g$, denoted $g_d$, that allows us to maximise $\val_g$ by minimising $\cost_{g_d}$.
Thereafter, we combine $\cost_{g_d}$ with the non-dual $\cost_s$ to make up $\cost_{f_d}$.

\medskip

For an ordered set $(X,\omega)$, we define \deft{the dual of the antisymmetrisation} of $\omega$ to be
the function
\begin{equation} \label{eqn:gd}
  g_d(x,y) \ = \ 1 - g(x,y).
\end{equation}
Notice that $g = 1-g_d$, meaning that the two functions are each others duals. Notice also that
$g_d : X \times X \to [0,2]$, eliminating the negative coefficients in the optimisation problem.

\begin{lemma} \label{lemma:dual-g}
  Let $(X,\omega)$ be given, and let $g_d$ be the dual of the antisymmetrisation of $\omega$.
  Then the function $\cost_{g_d} : \btrees{X} \to \R_+$ defined as
  \begin{equation} \label{eqn:gd-cost}
    \cost_{g_d}(T) \ = \ \sum_{x \le_T y} \card{T[x \lor y]} g_d(x,y)
  \end{equation}
  is dual to $\val_g$ in the sense that any tree that maximises $\val_g$ also
  minimises $\cost_{g_d}$.
\end{lemma}
\begin{proof}
  Recalling \citep[Theorem 3]{Dasgupta2016}, stating that
  \[
  \sum_{x \le_T y} \card{T[x \lor y]} \ = \ \frac{\card{X}^3-\card{X}}{3}
  \]
  for all trees $T \in \btrees{X}$, this yields
  \begin{align*}
    \cost_{g_d}(T)
    &= \sum_{x \le_T y} \card{T[x \lor y]} (1 - g(x,y)) \\
    &= \sum_{x \le_T y} \card{T[x \lor y]} - \sum_{x \le_T y}\card{T[x \lor y]}g(x,y) \\
    &= \frac{\card{X}^3-\card{X}}{3} - \val_g(T).
  \end{align*}
  Hence, any tree that maximises $\val_g$ is a tree that minimises~$\cost_{g_d}$.
\end{proof}

For an ordered similarity space $(X,s,\omega)$,
we define the \deft{dual split value function} $f_d : X \times X \to [0,3]$ as
\[
f_d(x,y) \ = \ 2 - f(x,y).
\]

\begin{theorem} \label{thm:cost-df}
  Given $(X,s,\omega)$ where $s$ is a similarity and $\omega$ is a relaxed order, the maximisation
  problem of Definition~\ref{def:maximisation-problem}
  can be solved by minimising
  \[
  \cost_{f_d}(T) \ = \ \sum_{x \le_T y} \card{T[x \lor y]} f_d(x,y).
  \]
\end{theorem}
\begin{proof}
  Since
  $f_d = 2-f = (1-s_d) + (1-g) = s + g_d,$
  it follows that $\cost_{f_d} = \cost_s + \cost_{g_d}$. If we let
  $M = \frac{\card{X}^3 - \card{X}}{3}$, we have
  $$2M - \cost_{f_d} = (M - \cost_s) + (M - \cost_{g_d}) = \val_{s_d} + \val_g = \val_f.$$
  Hence, a tree that maximises $\val_f$, also minimises $\cost_{f_d}$.
\end{proof}

\paragraph{A remark on the dual of the antisymmetrisation.}
The similarity $s$ can be defined in terms of an undirected graph with weight
function $s$, and the dual $s_d = 1-s$ is also a graph. Indeed, if $s$ has unit weights, then $s_d$ is the
graph with complementary edge set. However, for the order relation, we have a directed
weighted graph with weight function $\omega$, and with antisymmetrisation $g(x,y) = \omega(x,y) - \omega(y,x)$.
The dual of the antisymmetrisation $g_d = 1-g$ is dual to $g$, but does not correspond to an antisymmetrisation
of a relaxed order or weight function, since $g_d$ itself is not antisymmetric. It is the dual graph to
the complete graph over $X$ with weight function $g$, but what this dual graph represents is not obvious.

\subsection{Approximation algorithm}

This section describes the approximation algorithm, and provides an analysis of its
relative performance guarantee. In the algorithm, we assume the existence of a
\emph{directed cut approximation function}; that is, a function that that takes as arguments
a vertex set together with a weight function, and produces an approximation of an optimal cut that comes
with a relative performance guarantee.

Concretely, given an ordered similarity space $(X,s,\omega)$,
if $\cut$ is an approximation of \DirSparsestCut{} with a relative performance guarantee $\alpha_\cut$ and if
$\cut(X,f_d) = (A,B)$, then $f_d(A,B)/(\card{A}\card{B})$ is no more than a factor $\alpha_\cut$ off from
a true sparsest directed cut of $(X,f_d)$.

The currently best known approximation algorithm of \DirSparsestCut{}
has a relative performance guarantee of~$O(\sqrt{\log n})$~\citep{AgarwalEtAl2005}.

\begin{definition} \label{def:approx-algo}
  Let $(X,E)$ be a complete directed weighted graph with weight function $f_d$, and let
  $\cut$ be an approximation algorithm of \DirSparsestCut. The following algorithm produces
  the approximate order preserving hierarchical clustering of $(X,f_d)$.

  \medskip

  \vbox{%
    \begin{algorithmic}
      \Procedure{MakeTree}{$X,f_d$}
      \If{$\card{X}=1$}
      \State \Return $X$
      \Else
      \State Let $(A,B)$ be the cut approximation returned by $\cut(X,f_d)$
      \State Let \textproc{LeftTree} be the tree returned by \textproc{MakeTree}$(A,f_d)$
      \State Let \textproc{RightTree} be the tree returned by \textproc{MakeTree}$(B,f_d)$
      \State \Return (\textproc{LeftTree}, \textproc{RightTree})
      \EndIf
      \EndProcedure
    \end{algorithmic}%
  }%
\end{definition}

Given an approximation of \DirSparsestCut{} with an approximation guarantee
of $O(\sqrt{\log n})$, according to the following theorem, the above algorithm produces a binary tree
that has a cost that is no higher than a factor \bound{} of an optimal tree.

\begin{theorem} \label{thm:approx-bound}
  Given $(X,s,\omega)$, let $T^*$ be a tree minimising $\cost_{f_d}$, and let $T$ be the tree
  returned by the algorithm of Definition~\ref{def:approx-algo}, using a cut approximation function
  for \DirSparsestCut{} with a relative performance guarantee of~$\alpha_\cut$. Then
  \[
  \cost_{f_d}(T) \ \le \ \frac{27 \alpha_\cut \log n}{2}  \cost_{f_d}(T^*).
  \]
\end{theorem}

\begin{proof}[Proof (sketch)]
  Since a \DirSparsestCut{} is equivalent to a \SparsestCut{} for an undirected graph,
  the theorem is a generalisation of \citep[Theorem~$12$]{Dasgupta2016}. The proof of Dasgupta's
  Theorem~$12$ requires the existence of a split with a particular bound of the cut density,
  provided by \citep[Lemma~$11$]{Dasgupta2016}.
  By replacing Dasgupta's Lemma~$11$ with our Lemma~\ref{lemma:T-split-bound} (below), proving
  Theorem~\ref{thm:approx-bound} amounts to reproducing the steps in the proof of
  \citep[Theorem~$12$]{Dasgupta2016}, just keeping in mind that we have to treat every split
  as an oriented split.
  We therefore provide the required split, and refer to \citep{Dasgupta2016} for details
  regarding the proof of the theorem.
\end{proof}

\begin{lemma} \label{lemma:T-split-bound}
  Given $(X,s,\omega)$, for every binary tree $T \in \btrees{X}$, there exists a split $(A,B)$ of $X$ for which
  \[
  \frac{f_d(A,B)}{\card{A} \card{B}}\ \le \ \frac{27}{2 n^3} \cost_{f_d}(T).
  \]
\end{lemma}
\begin{proof}
  The key is to control the cardinalities of the components of the split $(A,B)$, while at the same time
  making sure that the split is ordered according to the order induced by the tree.

  For a tree $T \in \btrees{X}$, we define a \deft{maximum cardinality path} to be a path in $T$
  starting at the root, and for every node,
  the next node on the path is a child of the current node of maximal cardinality.
  Based on a maximum cardinality path $\{N_i\}_{i=1}^m$, we can create the required split over the
  following steps:
  \begin{enumerate}
  \item Construct two sequences
    $\{\mcal{L}_i\}_{i=1}^{m-1}$ and $\{\mcal{R}_i\}_{i=1}^{m-1}$ of subsets of $X$ as follows:
    For $1 \le i \le m-1$, let $\ell(N_i)$ and $r(N_i)$ be the left and right children of $N_i$,
    respectively. Now, if $\card{\ell(N_i)} < \card{r(N_i)}$,
    set $\mcal{L}_i=\ell(N_i)$ and $\mcal{R}_i=\emptyset$, otherwise set
    $\mcal{L}_i=\emptyset$ and $\mcal{R}_i=r(N_i)$. We refer to these sets as the
    \emph{$i$-th left split-off} and \emph{$i$-th right split-off}, respectively.

  \item Define the set sequences $\{\mcal{A}_k\}_{k=1}^{m-1}$ and
    $\{\mcal{B}_k\}_{k=1}^{m-1}$ as
    \begin{align*}
      \mcal{A}_k &= \bigcup_{i=1}^k \mcal{L}_i & \mcal{B}_k &= \bigcup_{i=1}^k \mcal{R}_i.
    \end{align*}
    We refer to $\mcal{A}_k$ as \emph{the $k$-th accumulated left split-off},
    and $\mcal{B}_k$ as \emph{the $k$-th accumulated right split-off}.

  \item Let $K$ be the smallest natural number for which
    $\max\{\card{\mcal{A}_K},\card{\mcal{B}_K}\} \ge n/3$.

  \item Finally, if $\card{\mcal{A}_K} \ge n/3$, set $A=\mcal{A}_K$ and $B=X-A$.
    Otherwise, set $B=\mcal{B}_K$ and $A=X-B$.
  \end{enumerate}

  We refer to $(A,B)$ as a \deft{balanced $T$-split}, and call $\{N_i\}_{i=1}^K$
  \deft{the head sequence} of the balanced $T$-split.
  The key points to take away from this construction are the following:

  \begin{itemize}
  \item[$\quad $b1.] $\card{N_K} \ge \frac{n}{3}$,
  \item[$\quad $b2.] $(a,b) \in A \times B \ \Rightarrow \ a \le_T b$,
  \item[$\quad $b3.] $\cup_{i=1}^K \ell(N_i) = A \ \text{and} \ \cup_{i=1}^K r(N_i) = B$.
  \end{itemize}

  {\flushleft We prove each property in turn:}

  {\bf b1:} Since $\max\{\card{\mcal{A}_{K-1}},\card{\mcal{B}_{K-1}}\} < \frac{n}{3}$, and since,
  by construction,
  $X = N_k \cup \mcal{A}_{K-1} \cup \mcal{B}_{K-1}$, and since these sets are disjoint,
  it follows that $\card{N_K} \ge \frac{n}{3}$.

  {\bf b2:} Notice that the $\mcal{L}_i$ and $\mcal{R}_i$ are left- and right
  children of the head sequence elements. If $x$ is the leaf of the maximum cardinality path, this means that
  $a \in \mcal{L}_i \Rightarrow a \le_T x$, and likewise, $b \in \mcal{R}_i \Rightarrow x \le_T b$, so
  $(\mcal{A}_K,\mcal{B}_K)$ is order preserving with respect to $\le_T$.
  Now, if $\card{\mcal{A}_K} \ge \frac{n}{3}$,
  this means that $A = \mcal{A}_K$ and $B = \mcal{B}_K \cup r(N_K)$, which again means that
  $(A,B)$ is order preserving with respect to $\le_T$.
  A symmetric argument covers the case $\card{\mcal{B}_K} \ge \frac{n}{3}$ .

  {\bf b3:} If $\card{\mcal{A}_K} \ge \frac{n}{3}$, then
  $\mcal{R}_K = \emptyset$, so
  \begin{align*}
    A &= \bigcup\nolimits_{i=1}^K \mcal{L}_i \ = \ \bigcup\nolimits_{i=1}^K \ell(N_i), \\
    B &= \left( \bigcup\nolimits_{i=1}^K \mcal{R}_i \right) \cup r(N_K) \ = \
    \left( \bigcup\nolimits_{i=1}^{K-1} r(N_i) \right) \cup r(N_K).
  \end{align*}
  And again, a symmetric argument attests ${\bf{b3}}$ for $\card{\mcal{B}_K} \ge \frac{n}{3}$.

  \bigskip

  That proves the properties of the balanced $T$-split, and we are now in position to
  prove the statement of the lemma:
  Let $(A,B)$ be a balanced $T$-split, and let $\{N_i\}_{i=1}^K$ be the head sequence
  of the split. Then
  \begin{align*}
    \cost_{f_d}(T)
    & \ \ge \ \sum_{i=1}^K \card{N_i} f_d(\ell(N_i),r(N_i))
    \ \ge \ \sum_{i=1}^K \frac{n}{3} f_d(\ell(N_i),r(N_i))
    \ \ge \ \frac{n}{3} f_d(A,B).
  \end{align*}
  From this, we deduce that
  \[
  \frac{f_d(A,B)}{\card{A}\card{B}} <
  \frac{3}{n} \cdot \frac{\cost_{f_d}(T)}{(2n/3)  (n/3)} =
  \frac{27}{2n^3} \cost_{f_d}(T).
  \]
\end{proof}

{\flushleft \bf Recovering intuition---\DirDensestCut:}
While the dual of the antisymmetrisation defies intuitive interpretation, we can recover our intuitive
understanding of the approximation as follows.
Assume that $(A,B)$ is a sparsest cut under $f_d$. Since
\begin{align*}
  \frac{f_d(A,B)}{\card{A}\card{B}}
  \ = \ \frac{\sum_{A \times B} 2 - f(a,b)}{\card{A}\card{B}}
  \ = \ \frac{2\card{A}\card{B} - f(A,B)}{\card{A}\card{B}}
  \ = \  2 - \frac{f(A,B)}{\card{A}\card{B}},
\end{align*}
it follows that $(A,B)$ maximises $f(A,B)/(\card{A}\card{B})$. It seems natural to
call a cut that maximises $f(A,B)/(\card{A}\card{B})$ a \deft{directed densest cut (under $f$)}.
From what the authors can find, this cut is not previously mentioned in the
literature,
but one may
still assume the existence of an approximation algorithm for \DirDensestCut.
It is, thus, clear that the above approximation algorithm of $\cost_{f_d}$ is equivalent to
successive applications of a \DirDensestCut{} approximation to approximate a maximal tree
under $\val_f$.

\section{Demonstration}
\label{section:demo}

In this section, we demonstrate the efficacy of the theory in recovering copy-paste relations
from the machine parts dataset described in Section~\ref{section:motivation}.

The demonstration makes use of the approximation scheme describe in Section~\ref{section:approx}.
As mentioned, there exist good approximation algorithms for \DirSparsestCut. However, as we discuss further
in Section~\ref{section:summary}, the availability of 
these approximations
is scarce, so for the demonstration, we have settled on optimal \DirSparsestCut.
The downside is that the algorithm can only deal with fairly small sets,
but, as we show below, the sample sizes still suffice to provide a valuable
demonstration. For industrial applications, a proper \DirSparsestCut{} approximation will be
required.

\subsection{A model for planted partitions over machine parts data}
\label{subs:model-for-planted-partitions}

The dataset we use is described in detail in the data article~\citep{BakkelundMachineParts},
and is the same dataset that was used for the
demonstration in~\citep{Bakkelund2021}. The data article is backed by an open source
implementation\footnote{\url{https://pypi.org/project/machine-parts-pp/}} that produces planted
partitions simulating the copy paste mechanism described in Section~\ref{section:motivation}.

In brief, the planted partitions are generated as follows: First, a sample $X_0$ of size $n$ is drawn from the
original data, giving us an ordered similarity space $(X_0,s_0,\le_0)$.
We then duplicate $(X_0,\le_0)$ a total of $m$ times.
These duplications represent the copy-paste relations. We denote the copies by
$(X_j,\le_j)$ for $1 \le j \le m$, and enumerate the sets as $X_j = \{x^j_1,\ldots,x^j_n\}$
for $0 \le j \le m$. This means that, for $1 \le i \le n$ we have that $x_i^j$ and $x_i^k$ are copies of
each other whenever $j \ne k$.
Now, set $X = \cup_{j=0,\ldots,m} X_j$, and let the order relation $\le$ on $X$ be the union of the $\le_j$,
so that $(X,\le)$ is a partially ordered set.
To define the similarity $s : X \times X \to [0,1]$,
we proceed as follows: If $x$ and $y$ are in the same copy-paste component, then
their similarity is the same as in the original component. And, if they
are in different copy-paste components, then their similarity is that of the original
similarity minus a stochastic value. Concretely, let $\mcal{N}(\mu,\sigma^2)$ be the
normal distribution located at $\mu$ with scale $\sigma^2$, and let $Y \sim \mcal{N}(\mu,\sigma^2)$.
We then define
\begin{equation} \label{eqn:s_mod}
s(x^i_p,x^j_q) =
\begin{cases}
  s_0(x^0_p,x^0_q) & \text{if $i=j$},\\
  s_0(x^0_p,x^0_q) - Y & \text{otherwise},
\end{cases}
\end{equation}
applying rejection sampling to ensure $s(x^i_p,x^j_q) \in [0,1]$.
In this way, the internal cohesion in the copy-paste components are maintained, while the
coupling between the copy-paste related items is reduced as $\mu$ and $\sigma^2$ increase.

\medskip

The above process leaves us with two pieces of data:
\begin{enumerate}
\item The ordered similarity space $(X,\le,s)$ to be clustered;
\item The hidden clustering (planted partition) representing the copy-paste relations,
  given by the clusters $C_i = \{x^j_i\}_{j=0}^m$ for $1 \le i \le n$.
\end{enumerate}

\subsection{Measures of clustering quality}

To evaluate the quality of a binary tree $T \in \btrees{X}$ relative a planted partition,
we follow the method of \citep[Sections~$8$ and~$9$]{Bakkelund2021}, applying the following
measures of clustering quality:

\paragraph{Adjusted Rand index.} Adjusted Rand indices are some of the most popular measures of clustering
quality, and the one-sided Rand indices are good indicators of the success in recovering planted
partitions.
To report to which degree the methods can recover the planted clustering,
we use the one-sided adjusted Rand index ($\ari$) assuming that the candidate clustering is drawn from the
family of all possible clusterings over the set \citep{GatesAhn2017}.

\paragraph{Element wise adjusted order Rand index.} \citet[\S 8.2.2]{Bakkelund2021} introduces a quality
measure to determine to which degree the induced order relation between two different clusterings of
the same ordered set coincide. We use this quality measure, $\oari$,
to decide to which degree the reported
clustering has an induced order relation that coincides with the induced order relation of the planted
partition.

\paragraph{Loops.} This quality measure computes the fraction of elements of $X$ participating in
cycles in the induced relation of the clustering. While simple, is pins down the error in
order preservation for a clustering. For easier comparison alongside the other measures,
we report on $1 - \text{fraction}$, so that $\loops=1$ if there are no cycles, and $\loops=0$ if all
elements participate in an cycle. The $\loops$ measure is described in detail
in \citep[\S 9.3]{Bakkelund2021}.

\medskip

Given a planted partition $\{C_j\}_j$ over $X$ and a tree $T \in \btrees{X}$, to compute the quality of
$T$ with respect to the above measures, we start by identifying the flat clustering of $T$ with the
highest $\ari$,\footnote{We employed the {\tt clusim python} library, version $0.4$, to compute the $\ari$.}
and then proceed to compute the other quality measures for this flat clustering.

Although this puts a higher emphasis on $\ari$ compared to the other measures, we would like to emphasise
that both $\oari$ and $\loops$ are central measures for order preserving clustering, and in particular
for applications where order preservation is key.

\subsection{Methods evaluated in the demonstration}

To demonstrate the method described in this paper, we use the approximation method from
Section~\ref{section:approx} to approximate a good tree for $\val_\alpha$ using optimal
\DirSparsestCut. We denote this method by $\hatval_\alpha$.

For comparison, we run a selection of other methods on the same problem instances. From \citep{Bakkelund2021},
we know that the method provided by that
article\footnote{We employed the {\tt ophac python} library version $0.4.5$ for this demonstration.}
is currently best in class for this type of
problems. This method is denoted by $\HCOP$, representing
a $30$-fold approximation using complete linkage \citep[Definition~14]{Bakkelund2021}.

Also, to compensate for the lack of hierarchical clustering methods that support do-not-cluster constraints,
\citet[\S 9.2]{Bakkelund2021} presents a method termed \emph{simulated constrained clustering}
based on the following trick. The initial similarity is modified so that every pair of comparable elements have
similarity zero. As demonstrated there, when used with classical hierarchical clustering and complete
linkage, this provides a significant improvement on the results, compared to just classical hierarchical
clustering and complete linkage without the trick.
We do the same here, modifying the similarity and approximating a good tree using
Dasgupta's method and no order relation. We denote the method by $\hatval_1^0$, since $\alpha=1$ means we
ignore the order relation, and the zero indicates that all comparable elements have zero similarity.
As we demonstrate below, the trick has a significant effect, but
it does not compete with the objective function we have presented in this work.

Finally, we also compare with classical hierarchical clustering with complete
linkage,\footnote{We employed {\tt scipy.cluster.hierarchy} from {\tt scipy} version $1.7.1$.}
denoting this method by $\HCCL$.


\subsection{Execution and results}
A number of experiments were run with different parameter settings and on different parts of the dataset.
While the magnitudes of the quality measures varied between the different parameter settings,
the trend was the same in all cases. Even the order of the methods with respect to quality of
results hardly changed between experiments. We have picked a subset of the experiments, presented
in Table~\ref{table:params}, to represent the overall observations.
For each parameter set in Table~\ref{table:params}, we ran $50$
experiments,\footnote{To ensure that convergence was reached with $50$ experiments, we ran additional
tests with $200$ experiments for a small selection of parameter combinations.}

where each experiment can be outlined as
\begin{enumerate}
\item Given the parameters, create a planted partition based on the procedure in Section~\ref{subs:model-for-planted-partitions};
\item Run each of the methods on the generated partially ordered similarity space;
\item Evaluate the output of each method according to the listed quality measures.
\end{enumerate}
The mean values reported in the below plots are the means over these $50$ executions.

\begin{table}[htpb]
  \begin{center}
    \begin{tabular}{cccp{7cm}}
      \hline
       par. & range & default & explanation \\
      \hline
      cc & $6$ & $6$ & connected component no. to draw from \\
      $n$ & $5$ & $5$ & size of drawn sample \\
      $m$ & $4$ & $4$ & number of copies to make \\
      $\alpha$ & $0 \cdots 1$ & best & weight of order relation vs similarity \\
      $\sigma^2$ & $0.05 \cdots 0.4$ & $0.15$ & variance of subtracted noise \\
      $\mu$ & $0.0 \cdots 0.4$ & $0.075$ & location of subtracted noise\\
      $\minDeg$ & $1.0$ & $1.0$ & minimum average degree of the drawn sample \\
      \hline
    \end{tabular}
    \caption{Parameters for the presented experiments. If the range column contains more than one value,
      this means that we present plots where this parameter varies over the specified range. In that case,
      all other parameters are set to the value in the default column. Notice that for $\alpha$, we have
      no default value, but present the result for the $\alpha$ with the best mean $\ari$ value.
    }
    \label{table:params}
  \end{center}
\end{table}

\subsubsection{Varying $\alpha$}
\label{section:varying-alpha}

%
%
\pgfplotscreateplotcyclelist{result cycle nomark}{%
  very thick,solid,black \\
  very thick,dashed,black\\
  solid,black\\
  dashed,black\\
}

\pgfplotscreateplotcyclelist{result cycle nomark loops}{%
  very thick,solid,black \\
  thin,color=gray!30!white \\
  color=gray!30!white \\
  very thick,dashed,black\\
  solid,black\\
  dashed,black\\
}

\pgfplotscreateplotcyclelist{result cycle mark}{%
  thick,solid,black,mark=oplus*,mark options={solid,scale=.5}\\
  thick,dashed,black,mark=oplus*,mark options={solid,scale=.5}\\
  solid,black,mark=oplus*,mark options={solid,scale=.5}\\
  dashed,black,mark=oplus*,mark options={solid,scale=.5}\\
}

\pgfplotscreateplotcyclelist{result cycle mark loops}{%
  thick,solid,black,mark=oplus*,mark options={solid,scale=.5}\\
  thin,color=gray!30!white \\
  color=gray!30!white \\
  thick,dashed,black,mark=oplus*,mark options={solid,scale=.5}\\
  solid,black,mark=oplus*,mark options={solid,scale=.5}\\
  dashed,black,mark=oplus*,mark options={solid,scale=.5}\\
}

\pgfplotsset{
  scale=0.5,
  every axis/.append style={
    x tick label style={
      /pgf/number format/.cd,
      fixed,
      fixed zerofill,
      precision=2,
      /tikz/.cd,
    },
    cycle list name=result cycle mark,
    line legend,
    legend to name=reallegend,
    legend columns=-1,
    legend entries= {
      $\hatval_\alpha\quad$,
      $\hatval_1^0\quad$,
      $\HCOP\quad$,
      $\HCCL$,
    },
    ylabel shift=-3pt,
  },
}

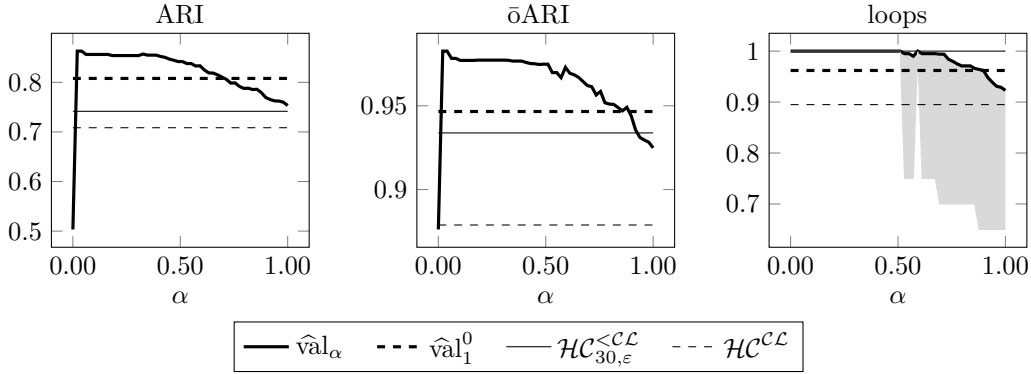
\begin{figure}[htpb]
  \begin{center}
    %
%
\begin{tabular}{ccc}
  $\phantom{AR}\ari$ & $\phantom{AR}\oari$ & $\phantom{AR}\loops$ \\
  \begin{tikzpicture}
    \begin{axis} [
        xlabel={$\alpha$},
        cycle list name=result cycle nomark,
      ]
      \addplot+ coordinates { (0.000000e+00,5.033545e-01) (2.041000e-02,8.629635e-01) (4.082000e-02,8.629635e-01) (6.122000e-02,8.562460e-01) (8.163000e-02,8.562460e-01) (1.020400e-01,8.562460e-01) (1.224500e-01,8.562460e-01) (1.428600e-01,8.562460e-01) (1.632700e-01,8.562460e-01) (1.836700e-01,8.540068e-01) (2.040800e-01,8.540068e-01) (2.244900e-01,8.540068e-01) (2.449000e-01,8.540068e-01) (2.653100e-01,8.540068e-01) (2.857100e-01,8.540068e-01) (3.061200e-01,8.540068e-01) (3.265300e-01,8.571416e-01) (3.469400e-01,8.549025e-01) (3.673500e-01,8.549025e-01) (3.877600e-01,8.549025e-01) (4.081600e-01,8.522155e-01) (4.285700e-01,8.504242e-01) (4.489800e-01,8.468415e-01) (4.693900e-01,8.446023e-01) (4.898000e-01,8.419154e-01) (5.102000e-01,8.419154e-01) (5.306100e-01,8.378849e-01) (5.510200e-01,8.378849e-01) (5.714300e-01,8.329587e-01) (5.918400e-01,8.334066e-01) (6.122400e-01,8.248977e-01) (6.326500e-01,8.190759e-01) (6.530600e-01,8.186281e-01) (6.734700e-01,8.150454e-01) (6.938800e-01,8.105671e-01) (7.142900e-01,8.060888e-01) (7.346900e-01,7.984757e-01) (7.551000e-01,7.989235e-01) (7.755100e-01,7.917582e-01) (7.959200e-01,7.877277e-01) (8.163300e-01,7.881755e-01) (8.367300e-01,7.854885e-01) (8.571400e-01,7.859364e-01) (8.775500e-01,7.792189e-01) (8.979600e-01,7.693666e-01) (9.183700e-01,7.653361e-01) (9.387800e-01,7.626491e-01) (9.591800e-01,7.622013e-01) (9.795900e-01,7.599621e-01) (1.000000e+00,7.532447e-01) };
      \addplot+ coordinates { (0.000000e+00,8.078801e-01) (1.000000e+00,8.078801e-01) };
      \addplot+ coordinates { (0.000000e+00,7.416010e-01) (1.000000e+00,7.416010e-01) };
      \addplot+ coordinates { (0.000000e+00,7.084615e-01) (1.000000e+00,7.084615e-01) };
    \end{axis}
  \end{tikzpicture}
  &
  \begin{tikzpicture}
    \begin{axis} [
        xlabel={$\alpha$},
        cycle list name=result cycle nomark,
      ]
      \addplot+ coordinates { (0.000000e+00,8.762099e-01) (2.041000e-02,9.826074e-01) (4.082000e-02,9.826074e-01) (6.122000e-02,9.782741e-01) (8.163000e-02,9.782741e-01) (1.020400e-01,9.770573e-01) (1.224500e-01,9.770573e-01) (1.428600e-01,9.770573e-01) (1.632700e-01,9.770573e-01) (1.836700e-01,9.772322e-01) (2.040800e-01,9.772322e-01) (2.244900e-01,9.772322e-01) (2.449000e-01,9.772322e-01) (2.653100e-01,9.772322e-01) (2.857100e-01,9.772322e-01) (3.061200e-01,9.772322e-01) (3.265300e-01,9.771031e-01) (3.469400e-01,9.766776e-01) (3.673500e-01,9.766776e-01) (3.877600e-01,9.766776e-01) (4.081600e-01,9.758508e-01) (4.285700e-01,9.752144e-01) (4.489800e-01,9.750408e-01) (4.693900e-01,9.746265e-01) (4.898000e-01,9.748044e-01) (5.102000e-01,9.748044e-01) (5.306100e-01,9.698044e-01) (5.510200e-01,9.698044e-01) (5.714300e-01,9.668055e-01) (5.918400e-01,9.729239e-01) (6.122400e-01,9.693980e-01) (6.326500e-01,9.684679e-01) (6.530600e-01,9.669146e-01) (6.734700e-01,9.644952e-01) (6.938800e-01,9.620260e-01) (7.142900e-01,9.613691e-01) (7.346900e-01,9.564460e-01) (7.551000e-01,9.585283e-01) (7.755100e-01,9.517947e-01) (7.959200e-01,9.510199e-01) (8.163300e-01,9.508279e-01) (8.367300e-01,9.489302e-01) (8.571400e-01,9.469738e-01) (8.775500e-01,9.490708e-01) (8.979600e-01,9.439734e-01) (9.183700e-01,9.355365e-01) (9.387800e-01,9.310626e-01) (9.591800e-01,9.295214e-01) (9.795900e-01,9.282612e-01) (1.000000e+00,9.248609e-01) };
      \addplot+ coordinates { (0.000000e+00,9.465722e-01) (1.000000e+00,9.465722e-01) };
      \addplot+ coordinates { (0.000000e+00,9.337947e-01) (1.000000e+00,9.337947e-01) };
      \addplot+ coordinates { (0.000000e+00,8.788597e-01) (1.000000e+00,8.788597e-01) };
    \end{axis}
  \end{tikzpicture}
  &
  \begin{tikzpicture}
    \begin{axis} [
        xlabel={$\alpha$},
        cycle list name=result cycle nomark loops,
        legend to name=hidelegend,
      ]
      \addplot+ [name path=A] coordinates { (0.000000e+00,1.000000e+00) (2.041000e-02,1.000000e+00) (4.082000e-02,1.000000e+00) (6.122000e-02,1.000000e+00) (8.163000e-02,1.000000e+00) (1.020400e-01,1.000000e+00) (1.224500e-01,1.000000e+00) (1.428600e-01,1.000000e+00) (1.632700e-01,1.000000e+00) (1.836700e-01,1.000000e+00) (2.040800e-01,1.000000e+00) (2.244900e-01,1.000000e+00) (2.449000e-01,1.000000e+00) (2.653100e-01,1.000000e+00) (2.857100e-01,1.000000e+00) (3.061200e-01,1.000000e+00) (3.265300e-01,1.000000e+00) (3.469400e-01,1.000000e+00) (3.673500e-01,1.000000e+00) (3.877600e-01,1.000000e+00) (4.081600e-01,1.000000e+00) (4.285700e-01,1.000000e+00) (4.489800e-01,1.000000e+00) (4.693900e-01,1.000000e+00) (4.898000e-01,1.000000e+00) (5.102000e-01,1.000000e+00) (5.306100e-01,9.950000e-01) (5.510200e-01,9.950000e-01) (5.714300e-01,9.900000e-01) (5.918400e-01,1.000000e+00) (6.122400e-01,9.950000e-01) (6.326500e-01,9.950000e-01) (6.530600e-01,9.950000e-01) (6.734700e-01,9.950000e-01) (6.938800e-01,9.940000e-01) (7.142900e-01,9.940000e-01) (7.346900e-01,9.840000e-01) (7.551000e-01,9.800000e-01) (7.755100e-01,9.750000e-01) (7.959200e-01,9.710000e-01) (8.163300e-01,9.710000e-01) (8.367300e-01,9.710000e-01) (8.571400e-01,9.660000e-01) (8.775500e-01,9.640000e-01) (8.979600e-01,9.620000e-01) (9.183700e-01,9.480000e-01) (9.387800e-01,9.390000e-01) (9.591800e-01,9.310000e-01) (9.795900e-01,9.290000e-01) (1.000000e+00,9.230000e-01) };
      \addplot+ [name path=B]coordinates { (0.000000e+00,1.000000e+00) (2.041000e-02,1.000000e+00) (4.082000e-02,1.000000e+00) (6.122000e-02,1.000000e+00) (8.163000e-02,1.000000e+00) (1.020400e-01,1.000000e+00) (1.224500e-01,1.000000e+00) (1.428600e-01,1.000000e+00) (1.632700e-01,1.000000e+00) (1.836700e-01,1.000000e+00) (2.040800e-01,1.000000e+00) (2.244900e-01,1.000000e+00) (2.449000e-01,1.000000e+00) (2.653100e-01,1.000000e+00) (2.857100e-01,1.000000e+00) (3.061200e-01,1.000000e+00) (3.265300e-01,1.000000e+00) (3.469400e-01,1.000000e+00) (3.673500e-01,1.000000e+00) (3.877600e-01,1.000000e+00) (4.081600e-01,1.000000e+00) (4.285700e-01,1.000000e+00) (4.489800e-01,1.000000e+00) (4.693900e-01,1.000000e+00) (4.898000e-01,1.000000e+00) (5.102000e-01,1.000000e+00) (5.306100e-01,7.500000e-01) (5.510200e-01,7.500000e-01) (5.714300e-01,7.500000e-01) (5.918400e-01,1.000000e+00) (6.122400e-01,7.500000e-01) (6.326500e-01,7.500000e-01) (6.530600e-01,7.500000e-01) (6.734700e-01,7.500000e-01) (6.938800e-01,7.000000e-01) (7.142900e-01,7.000000e-01) (7.346900e-01,7.000000e-01) (7.551000e-01,7.000000e-01) (7.755100e-01,7.000000e-01) (7.959200e-01,7.000000e-01) (8.163300e-01,7.000000e-01) (8.367300e-01,7.000000e-01) (8.571400e-01,7.000000e-01) (8.775500e-01,6.500000e-01) (8.979600e-01,6.500000e-01) (9.183700e-01,6.500000e-01) (9.387800e-01,6.500000e-01) (9.591800e-01,6.500000e-01) (9.795900e-01,6.500000e-01) (1.000000e+00,6.500000e-01) };
      \addplot+ fill between[of=A and B];
      \addplot+ coordinates { (0.000000e+00,9.620000e-01) (1.000000e+00,9.620000e-01) };
      \addplot+ coordinates { (0.000000e+00,1.000000e+00) (1.000000e+00,1.000000e+00) };
      \addplot+ coordinates { (0.000000e+00,8.950000e-01) (1.000000e+00,8.950000e-01) };
    \end{axis}
  \end{tikzpicture}
  \\
  \multicolumn{3}{c}{\begin{NoHyper}\ref{reallegend}\end{NoHyper}}
\end{tabular}    
    \caption{The mean performance of $\widehat{\val}_\alpha$ for different values of
      $\alpha$, plotted together with the mean values of the other methods. Notice that the other methods
      are constant with respect to~$\alpha$, explaining the horizontal plot lines.
      The bottom of the shaded region indicates the minimum observed value of the $\loops$ measure
      for $\hatval_\alpha$.}
    \label{fig:alphas}
  \end{center}
\end{figure}

Figure~\ref{fig:alphas} shows the effect of varying $\alpha$ through $[0,1]$.
Recall that for $\alpha=0$, it is only the order relation that determines the tree.
The abrupt change for arbitrary small $\alpha>0$ is due to tie resolution;
in the experiment, the order relation is binary, and all relations have the same weight.
This may in some cases lead to several equally valued, optimal binary trees. As soon as we
introduce the similarity, the ties are broken, and the number of viable solutions drops.
Additionally, since the similarity contains information identifying good trees, the average quality
increases. But, as $\alpha \to 1$, the information provided by the order relation is gradually phased out,
and the quality of the clustering decreases until it reaches the value on the right end of the plot.

Notice that for $\alpha=1$, the $\ari$, $\oari$ and $\loops$ values are those that would be obtained for
the Dasgupta objective function, since $\hatval_1 = \hatval_{s_d}$. From the plot, we see that
for $\ari$, the Dasgupta objective function outperforms $\HCOP$ with a small margin, even though
$\hatval_{s_d}$ does not take the order relation into account.

It is also interesting to notice that $\hatval_1^0$, which neither uses the order
relation, outperforms $\HCOP$ with significant margin with respect to $\ari$;
an observation that is consistent through
all the experiments we have done. We choose to conclude that these observations merely underline the
significance of the Dasgupta objective function.

\medskip

The $\loops$ measure deserves a separate comment. First, we note that $\HCOP$ has zero loops, which
is guaranteed by that method. Second, from Theorem~\ref{thm:partial-orders-optimal-trees}, we know
that $\val_{\alpha=0}$, has no loops for an optimal tree. But what we see from the plots is that even the
approximation $\hatval_\alpha$ returns nothing but loop-free trees for $\alpha$ significantly larger than
zero. This is a very promising observation, indicating that $\hatval_\alpha$ can be used for applications
were loops are unacceptable, such as classic use cases requiring acyclic partitioning of graphs. We note
that further research is required to establish for which $\alpha > 0$ and under which assumptions on the
order relation and the similarity such a guarantee can be set forth.

\medskip

It should also be mentioned that for all the experiments, the best results were obtained
for values of $\alpha$ close to zero. An open question is to whether this would also be the case for a
weighted partial order without uniform weights.

\subsubsection{Varying the variance}

\begin{figure}[htpb]
  \begin{center}
    %
%
\begin{tabular}{ccc}
  $\phantom{AR}\ari$ & $\phantom{AR}\oari$ & $\phantom{AR}\loops$ \\
  \begin{tikzpicture}
    \begin{axis} [
        xlabel={$\sigma^2$},
      ]
      \addplot+ coordinates { (5.000000e-02,9.117771e-01) (1.000000e-01,8.634113e-01) (1.500000e-01,8.629635e-01) (2.000000e-01,8.231064e-01) (3.000000e-01,7.295096e-01) (4.000000e-01,6.995048e-01) };
      \addplot+ coordinates { (5.000000e-02,8.974465e-01) (1.000000e-01,8.311674e-01) (1.500000e-01,8.078801e-01) (2.000000e-01,7.313009e-01) (3.000000e-01,5.987427e-01) (4.000000e-01,5.073850e-01) };
      \addplot+ coordinates { (5.000000e-02,8.804289e-01) (1.000000e-01,7.899669e-01) (1.500000e-01,7.416010e-01) (2.000000e-01,6.789046e-01) (3.000000e-01,5.481377e-01) (4.000000e-01,4.500625e-01) };
      \addplot+ coordinates { (5.000000e-02,8.311674e-01) (1.000000e-01,7.635448e-01) (1.500000e-01,7.084615e-01) (2.000000e-01,6.417345e-01) (3.000000e-01,5.172373e-01) (4.000000e-01,4.137881e-01) };
    \end{axis}
  \end{tikzpicture}
  &
  \begin{tikzpicture}
    \begin{axis} [
        xlabel={$\sigma^2$},
      ]
      \addplot+ coordinates { (5.000000e-02,9.917436e-01) (1.000000e-01,9.842895e-01) (1.500000e-01,9.826074e-01) (2.000000e-01,9.700360e-01) (3.000000e-01,9.614604e-01) (4.000000e-01,9.256976e-01) };
      \addplot+ coordinates { (5.000000e-02,9.739914e-01) (1.000000e-01,9.634359e-01) (1.500000e-01,9.465722e-01) (2.000000e-01,9.064737e-01) (3.000000e-01,8.455864e-01) (4.000000e-01,8.020938e-01) };
      \addplot+ coordinates { (5.000000e-02,9.758870e-01) (1.000000e-01,9.418762e-01) (1.500000e-01,9.337947e-01) (2.000000e-01,8.988539e-01) (3.000000e-01,8.225362e-01) (4.000000e-01,7.603052e-01) };
      \addplot+ coordinates { (5.000000e-02,9.223247e-01) (1.000000e-01,9.099199e-01) (1.500000e-01,8.788597e-01) (2.000000e-01,8.440934e-01) (3.000000e-01,7.904240e-01) (4.000000e-01,7.096999e-01) };
    \end{axis}
  \end{tikzpicture}
  &
  \begin{tikzpicture}
    \begin{axis} [
        xlabel={$\sigma^2$},
        cycle list name=result cycle mark loops,
        legend to name=hidelegend,
      ]
      \addplot+ [name path=A] coordinates { (5.000000e-02,1.000000e+00) (1.000000e-01,1.000000e+00) (1.500000e-01,1.000000e+00) (2.000000e-01,1.000000e+00) (3.000000e-01,1.000000e+00) (4.000000e-01,1.000000e+00) };
      \addplot+ [name path=B]coordinates { (5.000000e-02,1.000000e+00) (1.000000e-01,1.000000e+00) (1.500000e-01,1.000000e+00) (2.000000e-01,1.000000e+00) (3.000000e-01,1.000000e+00) (4.000000e-01,1.000000e+00) };
      \addplot+ fill between[of=B and B];
      \addplot+ coordinates { (5.000000e-02,9.810000e-01) (1.000000e-01,9.920000e-01) (1.500000e-01,9.620000e-01) (2.000000e-01,9.420000e-01) (3.000000e-01,9.280000e-01) (4.000000e-01,9.240000e-01) };
      \addplot+ coordinates { (5.000000e-02,1.000000e+00) (1.000000e-01,1.000000e+00) (1.500000e-01,1.000000e+00) (2.000000e-01,1.000000e+00) (3.000000e-01,1.000000e+00) (4.000000e-01,1.000000e+00) };
      \addplot+ coordinates { (5.000000e-02,9.220000e-01) (1.000000e-01,9.570000e-01) (1.500000e-01,8.950000e-01) (2.000000e-01,8.840000e-01) (3.000000e-01,9.160000e-01) (4.000000e-01,9.580000e-01) };
    \end{axis}
  \end{tikzpicture}
  \\
  \multicolumn{3}{c}{\begin{NoHyper}\ref{reallegend}\end{NoHyper}}      
\end{tabular}    
    \caption{The mean performances of the different methods for varying $\sigma^2$.
      The markers indicate for which values of $\sigma^2$ the experiments were conducted, and the other
      parameters are according to the default values of Table~\ref{table:params}.
      For $\hatval_\alpha$, each point in the plot
      corresponds to the optimal value over all $\alpha \in [0,1]$. Notice that for the $\loops$-plots,
      the plot for $\HCOP$ coincides with the plot of $\hatval_\alpha$. All trees returned by $\hatval_\alpha$
      were loop-free for the chosen $\alpha$, hence, compared to Figure~\ref{fig:alphas},
      there is no visible gray region in the plot.}
    \label{fig:scale}
  \end{center}
\end{figure}
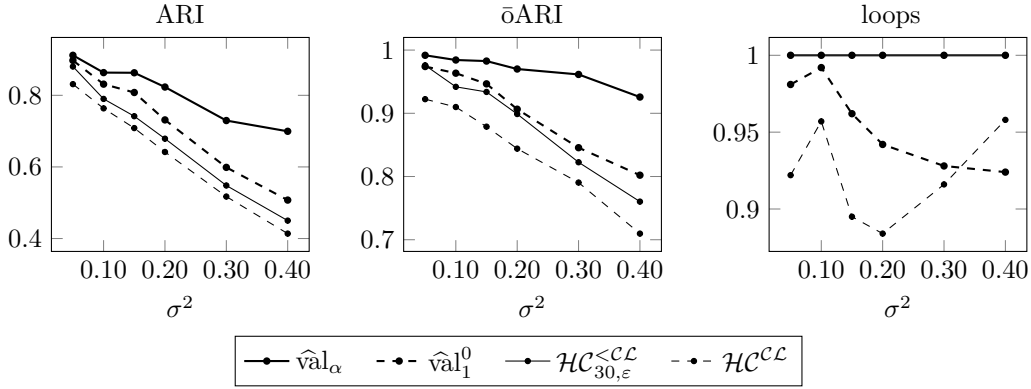

Figure~\ref{fig:scale} shows the effect of varying $\sigma^2$ through $[0.05,0.4]$; that is, the
variance of the noise added to the copy-paste similarities in~\eqref{eqn:s_mod}.
While each point in the plot of $\hatval_\alpha$
corresponds to the optimal value over all $\alpha \in [0,1]$, from what we see across the experiments,
more or less equally good results would be obtained by simply fixing $\alpha=0.1$ or some other low value (see
the comment at the end of Section~\ref{section:varying-alpha}).

As can be expected, all methods degrade in quality of the result when the variance in the similarities
increases, but $\hatval_\alpha$ is clearly less affected by this than the other methods.
Notice in particular that $\hatval_\alpha$ seems to be less affected by increase in variance
than $\hatval_1^0$, although they are based on very much the same optimisation principles.
At first, one may think that this is because the order relation is not modified by the noise, and this
information is available to $\hatval_\alpha$, but not to $\hatval_1^0$. However, the order relation is
utilised by $\HCOP$, and this method degrades just as severely as $\hatval_1^0$ when the variance
increases. One possible explanation is that it is the combined objective function, utilising both the
Dasgupta objective and the antisymmetrisation of the order relation~\eqref{eqn:g}
that causes $\hatval_\alpha$ to outperform the other methods.

Notice also that $\hatval_\alpha$ performs perfectly with respect to $\loops$,
regardless of the increasing noise variance.

\subsubsection{Varying translation magnitude}
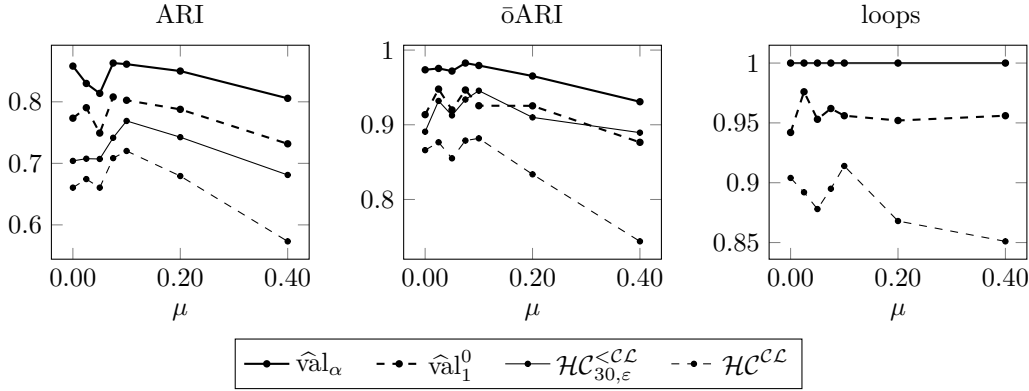
\begin{figure}[htpb]
  \begin{center}
    %
%
\begin{tabular}{ccc}
  $\phantom{AR}\ari$ & $\phantom{AR}\oari$ & $\phantom{AR}\loops$ \\
  \begin{tikzpicture}
    \begin{axis} [
        xlabel={$\mu$},
      ]
      \addplot+ coordinates { (0.000000e+00,8.580373e-01) (2.500000e-02,8.298239e-01) (5.000000e-02,8.137020e-01) (7.500000e-02,8.629635e-01) (1.000000e-01,8.611721e-01) (2.000000e-01,8.499763e-01) (4.000000e-01,8.056410e-01) };
      \addplot+ coordinates { (0.000000e+00,7.733971e-01) (2.500000e-02,7.904147e-01) (5.000000e-02,7.492142e-01) (7.500000e-02,8.078801e-01) (1.000000e-01,8.025062e-01) (2.000000e-01,7.877277e-01) (4.000000e-01,7.317487e-01) };
      \addplot+ coordinates { (0.000000e+00,7.039831e-01) (2.500000e-02,7.075658e-01) (5.000000e-02,7.071180e-01) (7.500000e-02,7.416010e-01) (1.000000e-01,7.689188e-01) (2.000000e-01,7.424967e-01) (4.000000e-01,6.811437e-01) };
      \addplot+ coordinates { (0.000000e+00,6.605435e-01) (2.500000e-02,6.744262e-01) (5.000000e-02,6.605435e-01) (7.500000e-02,7.084615e-01) (1.000000e-01,7.201051e-01) (2.000000e-01,6.793524e-01) (4.000000e-01,5.732162e-01) };
    \end{axis}
  \end{tikzpicture}
  &
  \begin{tikzpicture}
    \begin{axis} [
        xlabel={$\mu$},
      ]
      \addplot+ coordinates { (0.000000e+00,9.735965e-01) (2.500000e-02,9.754246e-01) (5.000000e-02,9.718947e-01) (7.500000e-02,9.826074e-01) (1.000000e-01,9.793329e-01) (2.000000e-01,9.652734e-01) (4.000000e-01,9.308359e-01) };
      \addplot+ coordinates { (0.000000e+00,9.135649e-01) (2.500000e-02,9.477701e-01) (5.000000e-02,9.198940e-01) (7.500000e-02,9.465722e-01) (1.000000e-01,9.255170e-01) (2.000000e-01,9.253516e-01) (4.000000e-01,8.765070e-01) };
      \addplot+ coordinates { (0.000000e+00,8.907460e-01) (2.500000e-02,9.319476e-01) (5.000000e-02,9.125255e-01) (7.500000e-02,9.337947e-01) (1.000000e-01,9.456596e-01) (2.000000e-01,9.098855e-01) (4.000000e-01,8.894620e-01) };
      \addplot+ coordinates { (0.000000e+00,8.660094e-01) (2.500000e-02,8.765598e-01) (5.000000e-02,8.550409e-01) (7.500000e-02,8.788597e-01) (1.000000e-01,8.818008e-01) (2.000000e-01,8.336466e-01) (4.000000e-01,7.439141e-01) };
    \end{axis}
  \end{tikzpicture}
  &
  \begin{tikzpicture}
    \begin{axis} [
        xlabel={$\mu$},
        cycle list name=result cycle mark loops,
        legend to name=hidelegend,
      ]
      \addplot+ [name path=A] coordinates { (0.000000e+00,1.000000e+00) (2.500000e-02,1.000000e+00) (5.000000e-02,1.000000e+00) (7.500000e-02,1.000000e+00) (1.000000e-01,1.000000e+00) (2.000000e-01,1.000000e+00) (4.000000e-01,1.000000e+00) };
      \addplot+ [name path=B]coordinates { (0.000000e+00,1.000000e+00) (2.500000e-02,1.000000e+00) (5.000000e-02,1.000000e+00) (7.500000e-02,1.000000e+00) (1.000000e-01,1.000000e+00) (2.000000e-01,1.000000e+00) (4.000000e-01,1.000000e+00) };
      \addplot+ fill between[of=B and B];
      \addplot+ coordinates { (0.000000e+00,9.420000e-01) (2.500000e-02,9.760000e-01) (5.000000e-02,9.530000e-01) (7.500000e-02,9.620000e-01) (1.000000e-01,9.560000e-01) (2.000000e-01,9.520000e-01) (4.000000e-01,9.560000e-01) };
      \addplot+ coordinates { (0.000000e+00,1.000000e+00) (2.500000e-02,1.000000e+00) (5.000000e-02,1.000000e+00) (7.500000e-02,1.000000e+00) (1.000000e-01,1.000000e+00) (2.000000e-01,1.000000e+00) (4.000000e-01,1.000000e+00) };
      \addplot+ coordinates { (0.000000e+00,9.040000e-01) (2.500000e-02,8.920000e-01) (5.000000e-02,8.780000e-01) (7.500000e-02,8.950000e-01) (1.000000e-01,9.140000e-01) (2.000000e-01,8.680000e-01) (4.000000e-01,8.510000e-01) };
    \end{axis}
  \end{tikzpicture}
  \\
  \multicolumn{3}{c}{\begin{NoHyper}\ref{reallegend}\end{NoHyper}}
\end{tabular}    
    \caption{The mean performances of the different methods for varying $\mu$.
      The markers indicate for which values of $\mu$ the experiments were conducted, and the other
      parameters are according to the default values of Table~\ref{table:params}.
      For $\hatval_\alpha$, each point in the plot
      corresponds to the optimal value over all $\alpha \in [0,1]$. Notice that for the $\loops$-plots,
      the plot for $\HCOP$ coincides with the plot of $\hatval_\alpha$. All trees returned by $\hatval_\alpha$
      were loop-free for the chosen $\alpha$, hence, compared to Figure~\ref{fig:alphas},
      there is no visible gray region in the plot.}
    \label{fig:location}
  \end{center}
\end{figure}

Figure~\ref{fig:location} shows the effect of varying $\mu$ through $[0,0.4]$; that is, the
location of the noise added to copy-paste similarities.
As before, we present the result for the optimal $\alpha$ for $\hatval_\alpha$.

We see that all methods, $\hatval_\alpha$ included, are equally affected by increasing translation of
similarities, but that $\hatval_\alpha$ performs consistently better than all the other methods throughout
all the experiments.

Also in this case, given a tree $T$ approximated by $\hatval_\alpha$, the flat clustering with the
best $\ari$ is loop-free in every experiment.

\section{Summary and future work} \label{section:summary}

We have defined the concept of order preserving hierarchical clustering and classified
the  binary trees over a set that correspond to order preserving hierarchical clusterings.
We have introduced the concept of a \emph{relaxed order relation}, where the relations are
weighted over $[0,1]$, presented an objective function for hierarchical clustering of relaxed orders,
and proven several beneficial properties of this function. In particular, whenever the relaxed order relation
is a uniformly weighted partial order relation, an optimal tree is always order preserving.
We have discussed the method in terms of bi-objective optimisation, where the order relation and
the similarity are convexly combined, and provided an efficient basis on which to explore the Pareto front.
We have provided a polynomial time approximation algorithm for the model, with a proven relative
performance guarantee of at least \bound, and, finally, we have demonstrated the method on the
machine parts dataset, showing that our method outperforms existing methods with significant margin.

\smallskip

The following future research topics have been mentioned in the paper:
\begin{itemize}
\item Theorem~\ref{thm:splitting-power} gives a probability bound for order preservation quality
  for $\val_g$, and \citep[Theorem $9$]{Dasgupta2016} gives a probability bound for clustering
  quality for $\val_{s_d}$. We suggest to establish a probability bound for
  the combined order preservation- and clustering quality for $\val_\alpha$.
\item According to Theorem~\ref{thm:partial-orders-optimal-trees}, if the order relation is a partial order,
  then the optimal trees for $\val_g$ are order preserving.
  It follows that the same result holds for $\val_\alpha$ when $\alpha=0$.
  However, from the demonstration in Section~\ref{section:varying-alpha}, we see that the
  approximation $\hatval_\alpha$ provides trees that are order preserving for $\alpha$
  significantly larger than $0$.
  For applications, a result that that can provide an order preservation guarantee for non-zero $\alpha$
  would be of great value. It would allow the optimisation to take both the order relation and the
  similarity into account, and at the same time guarantee order preservation.
\item The approximation bound \bound{} of $\hatval_\alpha$ (Theorem~\ref{thm:approx-bound}) can most likely
  be improved, in the same manner that has been achieved for the Dasgupta objective function
  approximation \citep{CharikarChatziafratis2017b,CohenAddadEtAl2019,RoyPokutta2017}.
\end{itemize}

Additionally, at the start of Section~\ref{section:demo}, we mentioned the lack of available implementations of
approximations of \DirSparsestCut. When trying to locate openly available implementations of such
algorithms, we found none. A complicating factor is that the theory behind the
\DirSparsestCut{} approximations is far from trivial, so implementing one's own is not generally
achievable unless already an expert in semidefinite programming. And finally, there are
no known non-trivial approximations of \DirSparsestCut{} that provide useful cuts.
We therefore suggest a call for action within operations research to make available realisations of the
existing approximation algorithms for \DirSparsestCut, at least for the scientific community, and for
not-unreasonably-small datasets.


\begin{acknowledgement}
  I would like to express gratitude for funding to the DataScience@UiO innovation cluster
(The Faculty of Mathematics and Natural Sciences, University of Oslo),
the Department of Informatics (The Faculty of Mathematics and Natural Sciences, University of Oslo),
and the SIRIUS Centre for Scalable Data Access (Research Council of Norway, project no.: 237898).
I would also like to thank Henrik Forssell, Department of Informatics (University of Oslo),
and Gudmund Hermansen, Department of Mathematics (University of Oslo), for invaluable comments,
questions and discussions leading up to this work.

\end{acknowledgement}
 
\bibliography{ms}

\end{document}